\documentclass[11pt]{article} 
\usepackage[left=1.2in, top=1in, bottom=1in, right=1.2in]{geometry}
\usepackage{times}
\usepackage{subcaption}
\usepackage{caption}

\usepackage[utf8]{inputenc} 
\usepackage[T1]{fontenc}    
\usepackage{hyperref}       
\usepackage{url}            
\usepackage{booktabs}       
\usepackage{amsfonts,amsmath,amssymb}       
\usepackage{nicefrac}       
\usepackage{microtype}      
\usepackage{xcolor}         
\usepackage{txfonts}
\usepackage{graphicx}
\usepackage{wrapfig,bm,comment,color}
\usepackage{breakurl,epsfig,epsf,fmtcount,semtrans,multirow,boldline}
\usepackage{tcolorbox}
\tcbuselibrary{skins}
\usepackage{tikz}
\usepackage{titletoc}
\usepackage{enumitem}
\usepackage{algorithm}
\usepackage{algpseudocode}
\usepackage[normalem]{ulem}

\usepackage[textsize=tiny]{todonotes}

\definecolor{darkred}{RGB}{150,0,0}
\definecolor{darkgreen}{RGB}{0,150,0}
\definecolor{darkblue}{RGB}{0,0,200}
\hypersetup{colorlinks=true, linkcolor=darkred, citecolor=darkgreen, urlcolor=darkblue}

\newtheorem{theorem}{Theorem}

\newtheorem{assumption}{Assumption}

\newtheorem{lemma}{Lemma}

\newtheorem{proposition}{Proposition}

\newcommand{\beq}{\begin{equation}}
\newcommand{\ba}{\begin{align}}
\newcommand{\ea}{\end{align}}

\newcommand{\eeq}{\end{equation}}

\def \endprf{\hfill {\vrule height6pt width6pt depth0pt}\medskip}

\newenvironment{proof}{\noindent {\bf Proof.} }{\endprf\par}

\newcommand{\vct}[1]{\bm{#1}}
\newcommand{\tr}[1]{\texttt{tr}(#1)}
\newcommand{\mtx}[1]{\bm{#1}}

\newcommand{\eps}{\varepsilon}

\newcommand{\mbar}{\text{-}}

\newcommand{\bxi}{\boldsymbol{\xi}}

\newcommand{\hb}{\vct{h}}
\newcommand{\hbb}{\bar{\vct{h}}}

\newcommand{\st}{\star}

\newcommand{\A}{{\mtx{A}}}

\newcommand{\Ub}{{\mtx{U}}}

\newcommand{\Lc}{{\cal{L}}}
\newcommand{\Lct}{\tilde{\cal{L}}}

\newcommand{\Nc}{{\cal{N}}}

\newcommand{\Dc}{{\cal{D}}}

\newcommand{\La}{{\boldsymbol{\Lambda}}}
\newcommand{\bmu}{{\boldsymbol{\mu}}}

\newcommand{\zerbb}{{\mathbf{0}}}
\newcommand{\Iden}{{\mtx{I}}}
\newcommand{\M}{{\mtx{M}}}

\newcommand{\Ec}{\mathcal{E}}
\newcommand{\Ac}{\mathcal{A}}
\newcommand{\Bal}{{\boldsymbol{\Delta}}}

\newcommand{\Rc}{\mathcal{R}}

\newcommand{\vb}{\vct{v}}

\newcommand{\Ic}{{\mathcal{I}}}

\newcommand{\all}{{\text{all}}}
\newcommand{\cb}{\mtx{c}}
\newcommand{\w}{\vct{w}}

\newcommand{\li}{\left<}
\newcommand{\ri}{\right>}

\newcommand{\ab}{\vct{a}}

\newcommand{\ub}{{\vct{u}}}

\newcommand{\g}{{\vct{g}}}

\newcommand{\Z}{\mtx{Z}}

\newcommand{\Xc}{\mathcal{X}}

\newcommand{\x}{\vct{x}}

\newcommand{\y}{\vct{y}}

\newcommand{\W}{\mtx{W}}

\newcommand{\Wc}{{\cal{W}}}

\newcommand{\Pc}{{\cal{P}}}
\newcommand{\X}{{\mtx{X}}}

\newcommand{\eb}{\vct{e}}

\newcommand{\R}{\mathbb{R}}

\renewcommand{\P}{\operatorname{\mathbb{P}}}
\newcommand{\E}{\operatorname{\mathbb{E}}}

\newcommand{\nn}{\nonumber}

\newcommand{\sft}[1]{\mathbb{S}(#1)}
\newcommand{\abs}[1]{\left|#1\right|}

\newcommand{\tn}[1]{\left\|{#1}\right\|_{\ell_2}}

\newcommand{\sgn}[1]{\textrm{sgn}(#1)}

\newcommand{\Wt}{\tilde{\mtx{W}}}

\newcommand{\Wb}{\mtx{\bar{W}}}

\newcommand{\att}{{\text{att}}}
\newcommand{\lda}{{\text{lda}}}

\newcommand{\todoasr}[1]{} 
\newcommand{\ankit}[1]{}
\newcommand{\asrnote}[1]{}
\newcommand{\YL}[1]{\textsf{\textcolor{orange}{[\textbf{YL}: #1]}}}

\usepackage{hyperref}
\usepackage{url}
\usepackage{natbib}

\title{When and How Unlabeled Data Provably Improve \\ In-Context Learning }

\author{\normalsize{Yingcong Li$^{1,4}$ ~~~\qquad Xiangyu Chang$^2$ ~~~\qquad Muti Kara$^3$~~ }\\ \normalsize{Xiaofeng Liu$^1$ \quad Amit Roy-Chowdhury$^2$ \quad Samet Oymak$^1$}\vspace{10pt}\\
\normalsize{$^1$University of Michigan \quad $^2$University of California, Riverside \quad $^3$Bilkent University \quad $^4$NJIT}}
\date{}
\begin{document}
\addtocontents{toc}{\protect\setcounter{tocdepth}{0}}
\maketitle

\begin{abstract} 
Recent research shows that in-context learning (ICL) can be effective even when demonstrations have missing or incorrect labels. To shed light on this capability, we examine a canonical setting where the demonstrations are drawn according to a binary Gaussian mixture model (GMM) and a certain fraction of the demonstrations have missing labels. We provide a comprehensive theoretical study to show that: (1) The loss landscape of one-layer linear attention models recover the optimal fully-supervised estimator but completely fail to exploit unlabeled data; (2) In contrast, multilayer or looped transformers can effectively leverage unlabeled data by implicitly constructing estimators of the form $\sum_{i\ge 0} a_i (\X^\top\X)^i\X^\top\y$ with $\X$ and $\y$ denoting features and partially-observed labels (with missing entries set to zero). We characterize the class of polynomials that can be expressed as a function of depth and draw connections to Expectation Maximization, an iterative pseudo-labeling algorithm commonly used in semi-supervised learning. Importantly, the leading polynomial power is exponential in depth, so mild amount of depth/looping suffices. As an application of theory, we propose looping off-the-shelf tabular foundation models to enhance their semi-supervision capabilities. Extensive evaluations on real-world datasets show that our method significantly improves the semisupervised tabular learning performance over the standard single pass inference.
\end{abstract}

\section{Introduction}

In-context learning (ICL) is an intriguing capability of modern language models and has enjoyed remarkable empirical success \citep{brown2020language,min2022rethinking}. This success is also being extended to multimodal scenarios \citep{zhou2024visual} as well as other modalities such as tabular data \citep{hollmann2022tabpfn}. The push toward test-time scaling and long-context models \citep{snell2024scaling,guo2025deepseek} has further boosted the benefits of ICL by allowing the model to ingest a large number of demonstrations. For instance, in ``Many-shot in-context learning'' paper, \cite{agarwal2024many} demonstrate that pushing more examples into context window can substantially boost the accuracy. MAPLE \citep{chen2025maple} improves many-shot ICL by pseudo-labeling high-impact unlabeled examples and incorporating them into the prompt. The many-shot ICL setting naturally raises the question of when and how ICL can succeed with weaker supervision. As we can harness longer context models to boost predictive accuracy, we may indeed run out of high-quality demonstrations with verified answers/chain-of-thoughts and may want to utilize weaker data sources. This motivates our central question:\vspace{2pt}
\[ 
\emph{\textbf{Q:}~When and how can transformers learn in context from unlabeled data?}
\]
We primarily investigate this question under a semisupervised ICL (SS-ICL) setting with Gaussian mixture models (GMMs). Formally, given a prompt containing a dataset of feature-label pairs $(\x_i, y_i)_{i=1}^n\in\R^d\times \R$ as demonstrations and a query feature $\x$ (see Eq. \eqref{def Z}), a model trained for ICL learns to predict the corresponding output $y$ given prompt. For ICL with a supervised binary GMM model, we have $\x_i\sim\Nc(\bmu_{y_i}, \sigma^2 \Iden)$ and $y_i\in \{-1,1\}$, $i\in[n]$, and the component means $\bmu_{\pm 1}$ that parameterize the classification task are sampled from a prior task distribution. This prompt model is well studied under various fully-supervised settings \citep{garg2022can,von2023transformers,ahn2024transformers,akyrek2023what,mahankali2024one,collins2024context,shen2024training} where each demonstration includes a clearly labeled output. In our SS-ICL setting, only $m$ out of $n$ total samples have correct labels ($m \le n$) either $-1$ or $1$, and remaining labels are unknown and fed to the model as $y_i=0$. 

In this work, we provide a comprehensive theoretical and empirical study of attention models with varying depths when trained with SS-ICL. Our analysis reveals the importance of \emph{depth}: Despite being able to implement the optimal fully-supervised estimator, single-layer linear attention completely fails to leverage unlabeled examples. In contrast, deeper or looped transformer architectures can emulate strong semi-supervision algorithms, approaching the performance of the Bayes-optimal classifier as depth increases. Informed by the importance of depth/looping, we also devise semisupervision strategies for tabular foundation models. Our specific contributions are: 

\begin{itemize}[label=$\diamond$]
\item \textbf{Landscape of one-layer linear attention ($\S$\ref{sec one layer}):} We study the optimization landscape of single-layer linear attention for the SS-ICL problem under an isotropic task prior. We prove that the global minimum of the loss function returns the plug-in estimator (see Eq. \eqref{pred spi}), i.e., $\hat{y} = \sgn{\x^\top \hat{\bmu}}$ with $\hat{\bmu}=\X^\top\y$, where $\X\in\R^{n\times d}$ represents features and $\y\in\R^n$ denotes partially-observed labels (with missing entries set to zero) of the ICL demonstrations. This implies that 1-layer model learns Bayes-optimal classifier in the fully-supervised setting, but completely fails to make use of unlabeled data.


\item \textbf{Depth is crucial but shallow can suffice (\S\ref{sec multilayer}):} We show that multilayer linear attention can emulate semisupervised learners by implementing polynomial estimators of the form 
\begin{align}
\hat{\bmu}=\sum_{i= 0}^K a_i (\X^\top\X)^i\X^\top\y.\label{main lin estimator}
\end{align}
Crucially, an $L$-layer (or looped) attention can express up to $K=O(3^L)$ powers, highlighting that logarithmic depth suffices to represent high-degree monomials. We provide characterizations of the set of expressible polynomials through different constructions (where each layer gets to update the features or labels of the previous layer). Corroborating these, experiments reveal that shallow models with $L\ge 2$ already achieve strong results and their performance can be approximately predicted through an eigen-estimator combining $i=0$ and $\infty$ (see \eqref{pred sspi}).  

\item \textbf{What learner attention emulates?} In Section \ref{em section}, we describe how each attention block can update the label estimates by emulating expectation-maximization (for linear attention) or belief propagation (for softmax attention). For instance \eqref{main lin estimator} can be interpreted as the model implicitly conducting an \emph{Expectation-Maximization} algorithm: Starting with the supervised estimator $\boldsymbol{\hat{\mu}}_0=\X^\top\y$, each term $(\X^\top\X)^i\X^\top\y$ can be viewed as a sequence of pseudo-labeling (expectation) $\hat{\y}_i=\X\boldsymbol{\hat{\mu}}_{i-1}$ and training (maximization) $\boldsymbol{\hat{\mu}}_i=\X^\top \hat{\y}_i$ steps. Corroborating this, we show that softmax-attention and softmax-transformer models similarly benefit from increasing depth and can emulate semisupervised learners competitive with Bayes limit (see Fig. \ref{fig tf}).

\item \textbf{Applications to Tabular FMs (\S\ref{sec exp}):} 
Tabular foundation models such as TabPFN \citep{hollmann2022tabpfn,hollmann2025accurate}, TabICL \citep{qu2025tabicl} and TabDPT \citep{ma2025tabdpt} represent a suitable application of theory as they also model the ICL examples with a single token. 
To harness unlabeled examples, we propose a novel strategy that iteratively creates soft pseudo-labels by \emph{explicitly looping the tabular FM} while controlling validation risk. Focusing on the few-shot learning setting where TabPFN-v2 \citep{hollmann2025accurate} excels, we demonstrate that our approach can significantly improve predictive performance on various real-world datasets. 
\end{itemize}

\subsection{Related Work}


\paragraph*{Theoretical Analysis of In-Context Learning} 
Recent work has developed theoretical frameworks for understanding in-context learning in transformers. \citet{akyrek2023what}, \citet{von2023transformers} and \citet{dai2023gpt} demonstrated that transformers emulate gradient descent during ICL. \citet{xie2022an} offered a Bayesian perspective, while \citet{zhang2024trained} showed transformers learn linear models in-context. \citet{ahn2024transformers} established they implement preconditioned gradient descent, and \citet{mahankali2024one} proved one-step gradient descent is optimal for single-layer linear attention. Multiple works \citep{li2023transformers,yang2024context,li2024fine,bai2023transformers,shen2024training} studied the generalization capability of transformers. However, these exclusively focus on fully-supervised settings, leaving a critical gap in understanding how transformers handle partially labeled data—a common real-world scenario. Our work addresses this gap by providing the first theoretical characterization of semi-supervised in-context learning. \cite{wangtheoretical} considers a setting where the model observes demonstrations of the form (query, response$_i$, reward$_i$) and aims to correct its response based on the reward sequence. Our work has a different focus as it highlights that the model can correct/impute the missing labels using implicit feedback from labeled demonstrations.

\paragraph*{Semi-Supervised Learning}
Traditional semi-supervised learning (SSL) aims to leverage unlabeled data to improve classifier performance. For linear classifiers, \citet{oymak2020statistical} characterized self-training iterations and demonstrated rejecting low-confidence samples; further theoretical analyses of self-training/pseudo-labeling cover deep networks \citep{wei2020theoretical}.
For Gaussian Mixture Models (GMMs), \citet{lelarge2019asymptotic} quantified maximal improvement from unlabeled data, while \citet{krishnapuram2004semi} developed graph-based priors. Learning GMMs via Expectation-Maximization (EM) or pseudo-labeling, especially with few labels, is well-studied. \citet{ratsaby1995learning} provided early PAC-style bounds for GMMs learned from few labeled and many unlabeled points. \citet{balakrishnan2017statistical} offered further statistical guarantees for EM. \citet{nigam2000text} demonstrated empirically that EM (viewable as iterative pseudo-labeling \cite{xu2024expectation}) with pseudo-labels significantly reduces text classification error using unlabeled documents. These foundational works, with ongoing research in areas like agnostic learning \citep{kwon2020algorithm}
underpin many SSL concepts. While these works established fundamental principles, they did not consider how these concepts apply to in-context learning with transformers. 
A most recent concurrent work \citep{liu2026unlabeled} makes a similar observation to ours, showing that softmax attention approximates an EM estimator in a sem-supervised ICL setting, but with a different focus on the underlying model and data regime.
Our contribution bridges this gap by showing how transformer depth enables effective utilization of unlabeled examples within the prompt, essentially implementing semi-supervised learning without parameter updates.


\section{Problem Setup and Preliminaries}\label{sec setup}

We study ICL in the setting of semi-supervised classification, where the in-context demonstrations are drawn from a binary Gaussian mixture model (GMM). We begin by introducing the following core notation:
%
Denote the set $\{1,2,\cdots,n\}$ as $[n]$ and use bold letters, such as $\x$ and $\X$, to represent vectors and matrices, respectively. Let $Q(\cdot)$ function return the right tail of the standard normal distribution. We use $\sgn{\cdot}$ denote the sign function which is defined as follows:
$
\sgn{x}=\begin{cases}
    1,&x\geq0\\
    -1,&x<0
\end{cases}.
$
\subsection{Semi-supervised Data Model}\label{sec data}
Consider a $d$-dimensional semi-supervised binary GMM with $n$ examples ${(\x_i, y_i)}_{i=1}^n$, where $\x_i \in \mathbb{R}^d$ denotes the feature vector and $y_i \in\{-1, 0, 1\}$ represents the corresponding observed label, with $y_i = 0$ indicating a missing label, and each label is revealed independently with probability $p \in [0,1]$. Specifically, the data is generated as follows (for each $i\in[n]$):
\begin{align}
\x_i = y^c_i \cdot \bmu + \bxi_i\quad ,\quad y_i=\begin{cases}
        y^c_i,&\text{w.p.}\quad p\\
        0,&\text{w.p.}\quad 1-p
    \end{cases}\quad\text{and}\quad \quad y_i^c=\begin{cases}
        1,&\text{w.p.}\quad 1/2\\
        -1,&\text{w.p.}\quad 1/2
    \end{cases}.
\label{def data}
\end{align}
Here $\bmu\sim\text{Unif}(\mathbb{S}^{d-1})$ denotes the task mean, which is sampled uniformly from the unit sphere, and  $\bxi_i \sim \mathcal{N}(0, \sigma^2 \Iden)$ is the random noise with $\sigma \geq 0$ being the noise level that controls the variability of $\x_i$ around its mean. $y_i^c $ denotes the true class label that is uniform over $\{-1,1\}$. Observe that $p=1$ corresponds to fully-supervised learning and $p=0$ corresponds to fully-unsupervised learning.

\subsection{In-context Learning and Linear Attention}\label{sec icl}
We build on the setting of \citep{garg2022can,mahankali2024one,zhang2024trained,li2024fine} and construct the in-context prompts with examples drawn from the model \eqref{def data} as follows. 
\paragraph*{Prompt Generation} 
Given a task vector $\bmu\sim\text{Unif}(\mathbb{S}^{d-1})$, we sample $(n+1)$ in-context demonstrations $(\x_i,y_i)_{i=1}^{n+1}$ according to \eqref{def data} and construct the prompt
\begin{align}
    \Z=\begin{bmatrix}
        \x_1&\x_2&\cdots&\x_n&\x\\
         y_1& y_2&\cdots& y_n& 0
    \end{bmatrix}^\top\in\R^{(n+1)\times (d+1)}. \label{def Z}
\end{align}  
We will investigate training a transformer such that given $\Z$ as prompt, it correctly predicts the label $y:=y_{n+1}^c$ of the query $\x:=\x_{n+1}$ through ICL.  

\paragraph*{Model Architecture} Our work primarily focuses on training of linear attention models. Given any prompt $\Z\in\R^{(n+1)\times (d+1)}$, which can be treated as a sequence of $(d+1)$-dimensional tokens, the linear attention mechanism outputs
\begin{align}
&\att(\Z;\Wc)=(\Z\W_{q}\W_{k}^\top\Z^\top)\M\Z\W_{v}\label{def att}
\end{align}
where $\Wc:=\{\W_{k},\W_{q},\W_{v}\in\R^{(d+1)\times(d+1)}\}$ denotes the set of  the key, query and value weight matrices. 
Therefore, given the prompt matrix $\Z\in\R^{(n+1)\times (d+1)}$ as input, the attention mechanism outputs a $(n+1)$-length sequence (i.e., $\att(\Z;\Wc)\in\R^{(n+1)\times(d+1)}$). 
Note that the label for the query $\x$ is excluded from the prompt $\Z$. 
Similar to \citet{ahn2024transformers}, we consider a training objective with a mask $\M=\begin{bmatrix}
    \Iden_n&0\\0&0
\end{bmatrix}$ to prevent input tokens from attending to the queries.
To ensure that all in-context examples are treated equally and that the model remains invariant to their order/position, we do not apply a causal mask following \cite{ahn2024transformers}. In contrast, \cite{li2025gating} explores the use of causal masking in multi-layer linear attention and analyzes its impact on the final prediction.

Building upon the single-layer linear attention mechanism of \eqref{def att}, we can extend our model to multiple layers to capture more complex patterns. Consider optimizing an $L$-layer linear attention model and let $\Z_\ell$ be the input of $\ell$th layer, $\ell\in[L]$. Additionally, let $\Wc_\ell:=\{\W_{k\ell},\W_{q\ell},\W_{v\ell}\in\R^{(d+1)\times(d+1)}\}$ be the corresponding weight matrices of $\ell$th layer. Then, recalling the attention mechanism \eqref{def att}, the input prompt of $\ell$th layer is defined by
\begin{align}
    \Z_\ell=\Z_{\ell-1}+\att(\Z_{\ell-1};\Wc_{\ell-1})\qquad\text{for}\qquad\ell=2,\ldots L,\label{def input ell}
\end{align}
and $\Z_1=\Z$. 
We focus on the next-token prediction setting, where the model makes a prediction based on the final query token $[\x^\top~0]^\top$. Let $\hb \in \mathbb{R}^{d+1}$ denote the linear prediction head. We define the output of the $L$-layer linear attention model at the last (query) token as
\begin{align}
f_{\att\text{-}L}(\Z) = \hb^\top \att(\Z_L; \Wc_L)_{[n+1]}. \label{def f att}
\end{align}
Recalling the sign function, the predicted label for $\x$ is given by 
$y_{\att\mbar L}(\Z)=\sgn{f_{\att\text{-}L}(\Z)}
$.


\paragraph*{Model Training}
With our attention-based architecture established, we now turn to the training procedure and evaluation metrics. 
Consider the ICL setting where each input prompt $\Z$ (cf.~\eqref{def Z}) corresponds to a randomly sampled task vector $\bmu\sim\text{Unif}(\mathbb{S}^{d-1})$ and let $\ell(\cdot):\R\to\R$ be the loss function. Additionally, define the set of attention weights $\Wc^{(L)}:=\cup_{\ell=1}^L\Wc_\ell\in(\R^{(d+1)\times(d+1)})^{ 3L}$. The objective of $L$-layer linear atention takes the following form:
\begin{align}
&\min_{\Wc^{(L)},\hb}\Lc_{\att\text{-}L}(\Wc^{(L)},\hb)\qquad\text{where}\qquad\Lc_{\att\text{-}L}(\Wc^{(L)},\hb)=
\E\left[\ell(y,f_{\att\text{-}L}(\Z))\right].\label{obj att}
\end{align}
Here, $y=y^c_{n+1}$ and the expectation subsumes the randomness of $\bmu$ and $(\bxi_i,y_i)_{i=1}^{n+1}$. The search space for $\Wc^{(L)}$ is $(\R^{(d+1)\times(d+1)})^{ 3L}$, and for $\hb$ is $\R^{d+1}$. 





\section{Loss Landscape of One-layer Linear Attention under SS-ICL}\label{sec one layer}
Previous work~\citep{ahn2024transformers,li2024fine,mahankali2024one} has shown that an optimized single-layer linear attention implements a form of preconditioned gradient descent over the linear in-context demonstrations provided within the prompt. However, to the best of our knowledge, prior studies have not addressed the semi-supervised setting, where some in-context labels are missing. In this section, we analyze the optimization behavior of single-layer linear attention under the semi-supervised binary GMM setting described in Section~\ref{sec setup}, and demonstrate that the single-layer model learns the optimal fully-supervised learner, but fails to utilize the unlabeled data.

We begin with the following optimal supervised label estimator under our problem setting.

\paragraph*{Supervised Plug-in (SPI) Estimator}
The plug-in method is a classical approach for supervised classification problems, aiming to find a linear combination of features that separates different categories. Under our problem setting, it also serves as the asymptotically Bayes-optimal estimator given only labeled data \citep{hastie2009elements,devroye2013probabilistic}.
Consider the binary semi-supervised GMM problem described in \eqref{def data} with dataset $(\x_i,y_i)_{i=1}^n$, and let $\Ic \subset [n]$ represent the indices of labeled samples, e.g., $y_i\neq0$ for $i\in\Ic$. The SPI estimator returns the task mean 
\begin{align}\label{pred spi}\tag{SPI}
     \hat\bmu_s=\frac{1}{|\Ic|}\sum_{i\in\Ic}y_i\x_i.
\end{align}
We next present the following theorem establishes that, under isotropic task prior, optimal single-layer linear attention is equivalent to the SPI estimation.
\begin{theorem} \label{thm one layer}
Let the prompt (cf.~\eqref{def Z}) be generated as described in Section~\ref{sec icl}. Consider the objective (cf.~\eqref{obj att}) with $L=1$ and squared loss function $\ell(y,\hat y)=(y-\hat y)^2$, and denote the optimal prediction as $y_{\att\text{-}1}^\st(\Z)$. Let $\hat\bmu_s$ represent the SPI estimator defined in \eqref{pred spi}.
Then, for any $\Z$ from \eqref{def Z}, we have 
\begin{align}
{y}_{\att\text{-}1}^\st(\Z) = \sgn{\x^\top\hat\bmu_s}.\label{1 layer pred}
\end{align}
Additionally, its classification error obeys
\begin{align}        
\P( {y}_{\att\text{-}1}^\st(\Z)\neq y)&=\E_{g\sim\Nc(0,1),h\sim\Xc^2_{d-1}}\left[Q\left(\frac{1+\eps_{\sigma}g}{\sigma\sqrt{(1+\eps_{\sigma}g)^2+\eps_{\sigma}^2h}}\right)\right]\label{one layer err}\\
        &\leq Q\left(\frac{1-10d\eps_{\sigma}^2}{\sigma}\right)+e^{-d}+e^{- 1/8\eps_{\sigma}^2}\nn
    \end{align}
    where we define $\eps_\sigma=\sigma/\sqrt{np}$ and $\Xc^2_d$ defines chi-squared distribution with $d$ degrees of freedom.
\end{theorem}


The proof of Theorem~\ref{thm one layer} is deferred to Appendix~\ref{app sec one layer}.
Eq.~\eqref{1 layer pred} shows that one-layer linear attention model indeed implements the optimal supervised predictor, assuming access to $np$ labeled examples. Therefore, the classification error corresponds exactly to that of the SPI estimator. The supervised classification problem has been extensively studied \citep{bartlett2006convexity,belkin2018overfitting,montanari2019generalization,thrampoulidis2020theoretical,chatterji2021finite,cao2021risk,wang2022binary,deng2022model}, with most existing work focusing on a single classification task in asymptotic data or overparameterized regimes.  In contrast, within the ICL framework considered in our setting, the task mean $\bmu$ is randomly sampled, and the classification error is computed by averaging over random draws of $\Z$, $y$, and $\bmu$.  Accordingly, in \eqref{one layer err}, we express the error in a simplified form as an expectation.

The experimental results in Figure~\ref{fig multi layer} support Theorem~\ref{thm one layer}, where dark blue circular markers represent the performance of the single-layer linear attention model, blue curves show the classification accuracy of the SPI estimator, and the red dotted curves depict the accuracy $1 - \P( {y}_{\att\text{-}1}^\st(\Z)\neq y)$ as computed from \eqref{one layer err}. The alignments of these curves empirically validate Theorem~\ref{thm one layer}. Implementation details and further discussion are provided in Section~\ref{sec exp}. Based on these results, we reach the following conclusion:\vspace{3pt}
\[ 
\hspace{0pt}\emph{1-layer linear attention learns optimal supervised estimator but doesn't benefit from unlabeled data.}
\]

As shown in Figs~\ref{fig diff n} and \ref{fig diff n 10000}, when the number of labeled samples ($np = 10$) is fixed, increasing the number of unlabeled examples (even up to $\sim\!10000$) has no effect on performance, as the dark blue markers remain at the same level. 

At first glance, this may seem counterintuitive—while the data is unlabeled, it still contains information about the classification feature. For instance, the mean of the data points carries relevant information, and one might expect the model to extract and leverage this for better predictions. This expectation is particularly reasonable when a large amount of unlabeled data is available, as the sample covariance matrix approximates the population covariance, i.e., $\E[\X^\top\X/n]=\bmu\bmu^\top+\sigma^2\Iden$ where $\X=[\x_1,\x_2,\cdots,\x_n]^\top\in\R^{n\times d}$.
%
%
%
%
%
%
The key insight into why single-layer attention fails to leverage unlabeled data lies in the expectation structure. In our isotropic GMM setting where $\bmu\sim\text{Unif}(\mathbb{S}^{d-1})$, the sample covariance matrix converges to $\E[\X^\top\X/n] =\E[\bmu\bmu^\top]+\sigma^2 \Iden=(1/d+\sigma^2)\Iden$, which contains no task-specific information. The expectation across multiple tasks loses the signal from $\bmu$. This explains why single-layer attention, operating in a meta-learning framework across many tasks rather than optimizing for a single fixed task, cannot extract useful information from unlabeled data.

In the following section, we study multi-layer linear attention and demonstrate that it has the ability to propagate $\X^\top\X$ into deeper layers, thereby enabling the model to utilize the unlabeled data. 

\begin{figure}[!t]
\centering
\begin{subfigure}{0.325\linewidth}    
\centering
\includegraphics[width=\linewidth]{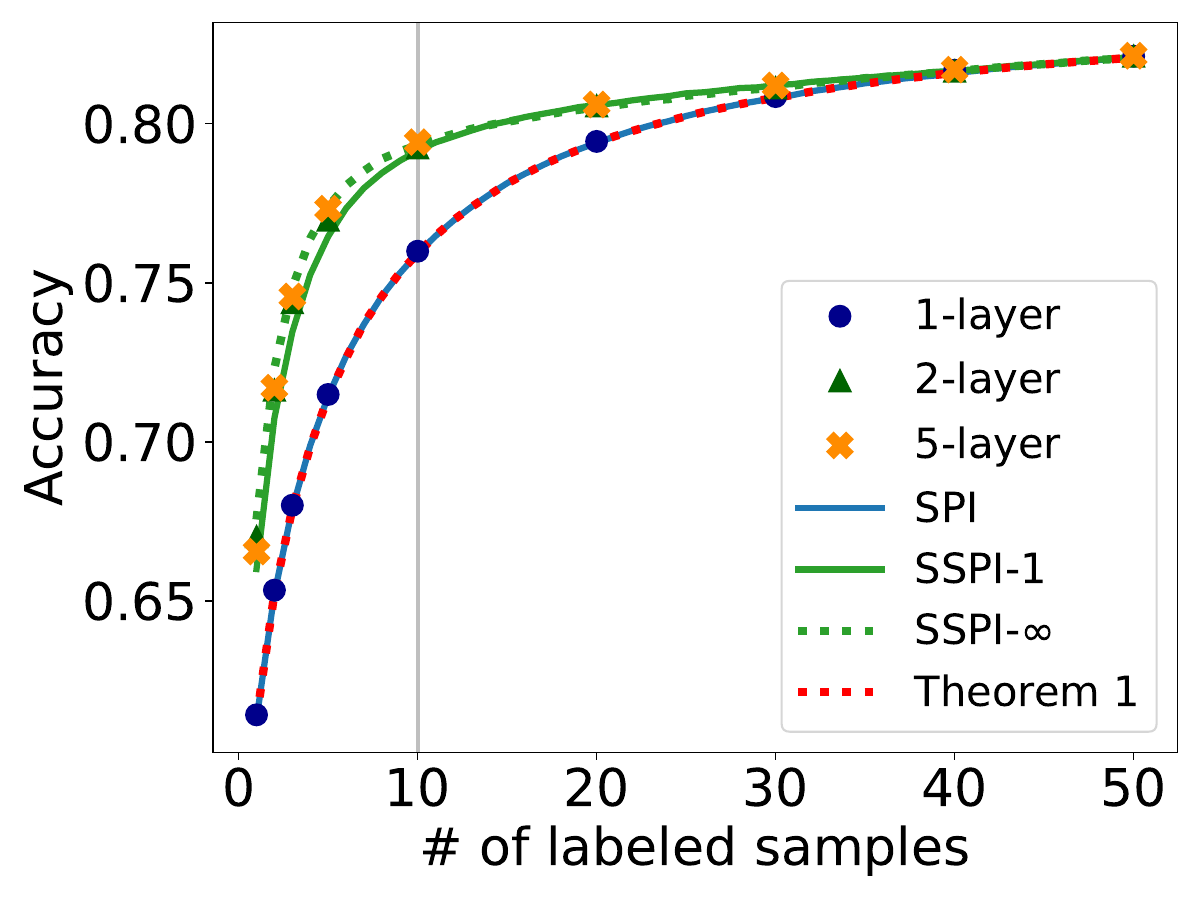}
\caption{$n=50$, $p \in (0, 1]$.}\label{fig diff m}
\end{subfigure}
\begin{subfigure}{0.325\linewidth}    
    \centering
    \includegraphics[width=\linewidth]{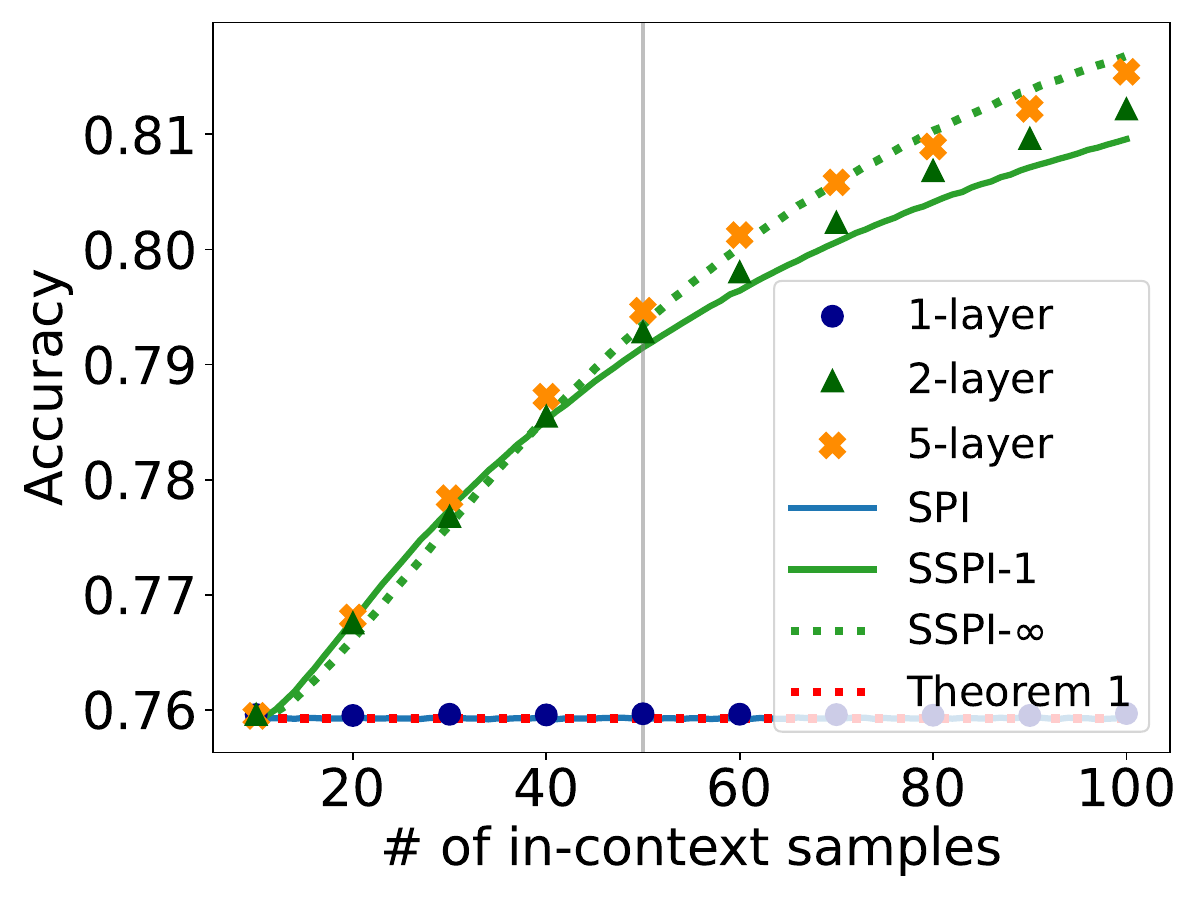}
\caption{$np=10$, $n \in \{10,\cdots,100\}$.}\label{fig diff n}
\end{subfigure}
\begin{subfigure}{0.325\linewidth}    
    \centering
    \includegraphics[width=\linewidth]{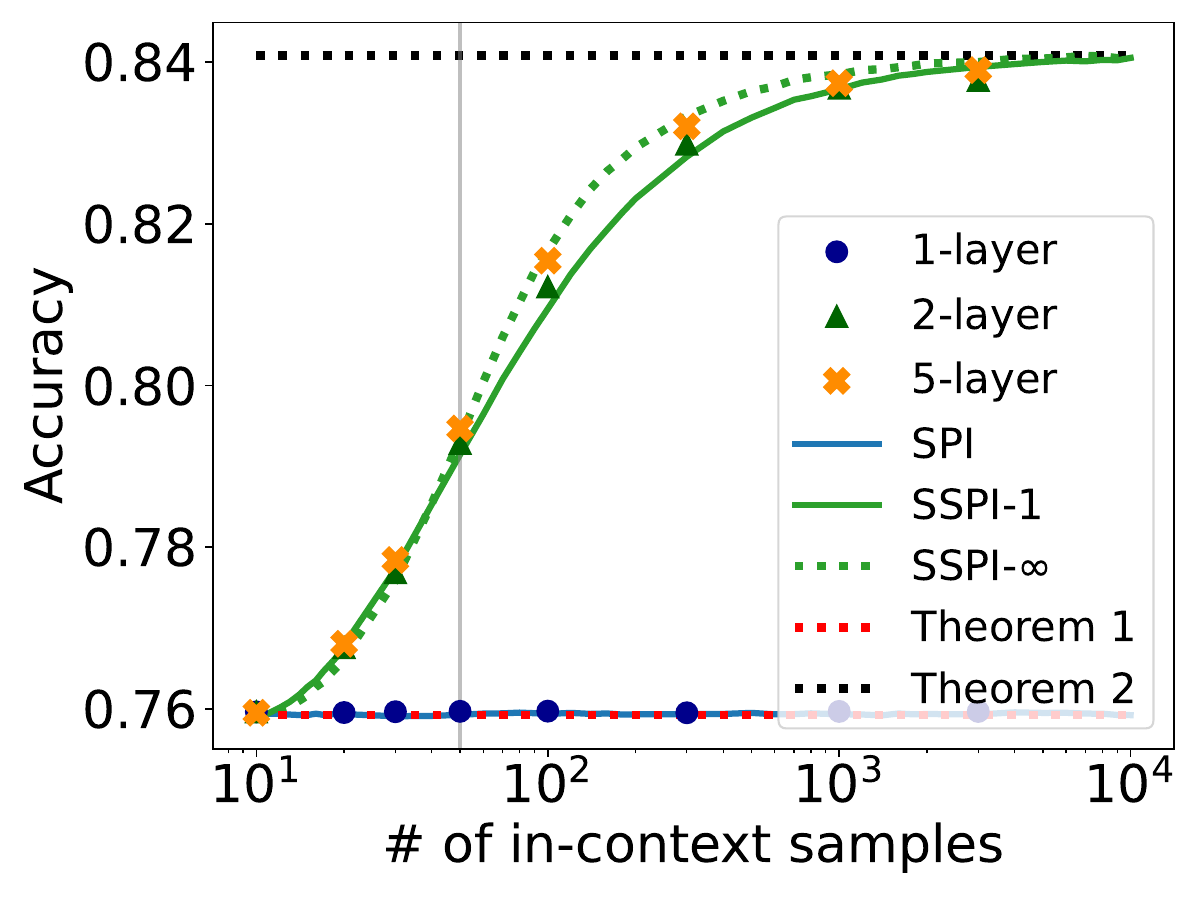}
\caption{$np=10$, $n \in \{10,\cdots,10000\}$.}\label{fig diff n 10000}

\end{subfigure}

\caption{\small{Experimental results support our theoretical findings presented in Sections~\ref{sec one layer} and~\ref{sec multilayer}.  
In all three subfigures, blue, green, and orange markers represent the results of 1-, 2-, and 5-layer linear attention models, respectively. The SPI estimator (cf.~\eqref{pred spi}), SSPI-$1$, and SSPI-$\infty$ (cf.~\eqref{pred sspi}) are shown as blue solid, green solid, and green dotted curves, respectively. 
The red dotted curves in all subfigures correspond to the single-layer/SPI results described in Eq.~\eqref{one layer err} of Theorem~\ref{thm one layer}, while the black dotted line in Fig. \ref{fig diff n 10000} corresponds to Eq.~\eqref{n infty err} of Theorem~\ref{thm optimal A}. Additional details and discussion can be found in Sections~\ref{sec one layer}, \ref{sec multilayer}, and \ref{sec exp}.
}}
\label{fig multi layer}
\end{figure}

\section{{Multi-layer Attention and the Benefits of Depth}}\label{sec multilayer}
In this section, we explore how deeper attention models can effectively utilize the unlabeled data.
Let 
\begin{align}\label{def X y}
    \X=\begin{bmatrix}
        \x_1&\x_2&\cdots&\x_n
    \end{bmatrix}^\top\in\R^{n\times d}\quad\text{and}\quad\y=\begin{bmatrix}
        y_1&y_2&\cdots&y_n
    \end{bmatrix}^\top\in\R^n.
\end{align}
\subsection{$L$-layer Linear Attention can Implement Degree-$O(3^L)$ Polynomials in $X^\top X$}\label{sec multi poly}
We first present the following propositions to show that multi-layer as well as looped linear attention can be expressed as a polynomial function of $\X^\top \X$. This structure allows the models to leverage unlabeled data to improve the estimation of the task mean $\bmu$.
\begin{proposition}\label{prop multilayer}
    Given an $L$-layer linear attention model described in Section~\ref{sec icl} with input prompt $\Z$ defined in \eqref{def Z}, one can construct the key, query, value weight matrices and the linear prediction head such that the model outputs (cf.~\eqref{def f att})
    \begin{align}
        f_{\att\text{-}L}(\Z)=\x^\top\A\X^\top\y.\label{pred L layer}
    \end{align}
    Then, the following $\A$ matrices are achievable via label and feature updates:
    \begin{itemize}
        \item \textbf{Label propagation:} $\A=c\prod_{\ell=1}^{L-1}\left(\Iden+c_\ell\X^\top\X\right)$ for arbitrary constants $\{c,c_1,\cdots,c_{L-1}\}$;
        \item \textbf{Feature propagation:} $\A=c\left(\X^\top\X\right)^{3^{L-1}-1}$ for an arbitrary constant $c$.
    \end{itemize}
\end{proposition}
\begin{proposition}\label{prop looped}
    Consider the same setting as in Proposition~\ref{prop multilayer}. There exists a single-layer linear attention model whose parameters can be constructed
    {such that, when looped $L$ times, its output reproduces that of \eqref{pred L layer},}
    with \( c_\ell \equiv c' \) for some arbitrary constant \( c' \).
\end{proposition}
The proofs of Proposition~\ref{prop multilayer} and ~\ref{prop looped} are deferred to Appendix \ref{app sec prop multi} and \ref{app sec prop looped}. In the following, we provide further clarification on the label and feature propagation.
\begin{enumerate}[leftmargin=*]
\item The final prediction of the label propagation process can be rewritten as
\begin{align*}
f_{\att\mbar L}(\Z) = c\x^\top \X^\top {\y}_L\quad\text{where}\quad\y_{\ell+1} = (\Iden + c_\ell \X \X^\top) \y_\ell,\quad \text{for}\quad\ell \in [L-1]
\end{align*}
with $\y_1 = \y$. Here, $\y_\ell$ can be interpreted as the soft pseudo-labels input to the $\ell$th layer, and each $c_\ell$ is parameterized by the attention mechanism in the corresponding layer. Although not exactly equivalent, the $L$-layer linear attention process shares similarities with the Expectation-Maximization (EM) algorithm for semi-supervised learning, with $L$ iterations of pseudo-labeling and a different label update strategy.

\item In contrast, the feature propagation process yields the final prediction
\begin{align*}
f_{\att\mbar L}(\Z)=c\x_L^\top\X_L^\top\y~~\text{where}~~\X_{\ell+1}=(\X_{\ell}\X_\ell^\top)\X_\ell~~\text{and}~~ \x_{\ell+1}=(\X_\ell^\top\X_\ell)\x_\ell,~~ \text{for}~~\ell \in [L-1]
\end{align*}
with $\X_1=\X$ and $\x_1=\x$. Here, $(\X_\ell, \x_\ell)$ can be viewed as the input features at the $\ell$th layer, encoding exponentially higher-order powers of $\X^\top \X$. This result highlights that a linear attention model requires only $O(\log K)$ layers to represent polynomial functions of degree $K$.
\end{enumerate}
Our construction for \emph{label propagation} is inherently related to the \emph{gradient descent} emulation capability of linear attention \cite{ahn2024transformers}. However, the \emph{feature propagation} construction is fundamentally different and underscores the transformer's capability to implement rapid power iteration over the empirical covariance $\X^\top\X$. 
In the above constructions, each attention block with residual connections updates features or labels using one parameter, namely mappings of the form $\X\rightarrow \X+\alpha \X\X^\top \X$ or $\y\rightarrow \y+\beta \X\X^\top\y$. The lemma below shows that, even if the multilayer model can express polynomials of $\X^\top\X$ with exponential degrees in depth, the expressible manifold of polynomials has dimensionality linear in depth.
\begin{lemma}[{Label $+$ Feature Propagation}]\label{lemma label+feature} For an $L$-layer linear attention model, the resulting eventual prediction corresponds to the matrix $\A$ in Proposition~\ref{prop multilayer} of the form
\begin{align}\label{A highest degree}
    \A=\sum_{\ell=0}^{(3^L-3)/2}a_\ell(\X^\top\X)^\ell.
\end{align}
The coefficients $\ab := [a_0~~a_1~~\cdots~~a_{(3^L-3)/2}]^\top$ lie on a manifold of dimension at most $2L$ as $\ab$ can be expressed as $\ab = g(\cb)$ for some smooth function $g : \mathbb{R}^{2L} \to \mathbb{R}^{(3^L - 3)/2}$ with $\cb$ representing the parameters of individual layers. 
\end{lemma}

\subsection{Which Semi-supervised Algorithm Does Multi-layer Attention Approximate?}

Recall the SPI estimator $\hat{\bmu}_s$ from \eqref{pred spi}, and that $\y$ denotes the visible labels defined in Section~\ref{sec data} and \eqref{def X y}. We have $\hat{\bmu}_s = \frac{1}{|\Ic|} \X^\top \y$. Motivated by Proposition~\ref{prop multilayer} that multi-layer linear attention can implement higher-degree polynomials of $\X^\top \X$, we introduce the following SSPI estimator, which makes predictions based on the supervised estimate $\hat{\bmu}_s$ combined with higher-order debiased term of the form $(\X^\top \X/n-\sigma^2\Iden)^k$.
\paragraph*{Semisupervised Plug-in (SSPI) Estimator}
Observe that the feature covariance satisfies $\E[\X^\top\X]/n=\bmu\bmu^\top+\sigma^2\Iden$, and the top eigenvector of the centered covariance matrix $(\X^\top \X/n - \sigma^2 \Iden)$ asymptotically aligns with either $\bmu$ or $-\bmu$. Therefore, with a substantial amount of unlabeled data, 
we propose the semisupervised plug-in (SSPI) estimator as follows:
\begin{align}\label{pred sspi}\tag{SSPI-$k$}
     \hat\bmu_{ss\mbar k}=\alpha\hat\bmu_s+(1-\alpha)
     (\X^\top\X/n-\sigma^2\Iden)^{k}\hat{\bmu}_s 
\end{align}
where $\hat\bmu_s$ is the SPI estimator (cf.~\eqref{pred spi}), and $\alpha\in[0,1]$ controls the trade-off between the fully-supervised and semi-supervised estimators. The optimal choice of $\alpha$ depends on the problem parameters $n,d$ and $p$.
Note that as $k \to \infty$, the term $(\X^\top \X / n - \sigma^2 \Iden)^k$ converges (up to scaling) to a rank-one projection onto the top eigenvector of the debiased covariance matrix, effectively serving as an estimator for $\bmu$ (up to sign).

In Figure~\ref{fig multi layer}, we present the prediction accuracies of 2-layer and 5-layer linear attention models, shown by green and orange markers, respectively. We also evaluate the SSPI algorithm with varying $k$ values, where the green solid curve corresponds to SSPI-$1$, and the green dotted represents SSPI-$\infty$, both using their respective optimal choices of $\alpha$. Details on selecting the optimal $\alpha$ values are provided in Section~\ref{sec exp} and illustrated in Figure~\ref{fig alpha}.  The results reveal a close alignment between multi-layer linear attention and SSPI estimators. Notably, the 2-layer model outperforms SSPI-$1$, due to its ability to implement higher-degree polynomials of $\X^\top \X$ (cf. Proposition~\ref{prop multilayer} and Equation~\eqref{A highest degree}).
When the sample size is sufficiently large (e.g., $n > 50$ in Figure~\ref{fig diff n}), the top eigenvector provides a more accurate estimate of the task mean, enabling SSPI-$\infty$ to achieve higher accuracy. Furthermore, since the 5-layer model is capable of representing higher-order functions than the 2-layer model, it can better estimate the top eigenvector, resulting in performance that closely matches that of SSPI-$\infty$.

In the following, we analyze the optimal classifier of the form $\sgn{\x^\top \A \hat{\bmu}_s}$ for a GMM, and provide insights into its behavior in the asymptotic regime as $n \to \infty$.
\begin{theorem}\label{thm optimal A}
    Consider a binary GMM defined in Section~\ref{sec data} and suppose that $(\x_i,y_i)_{i=1}^{n+1
    }$ is generated using a fixed $\bmu$ following \eqref{def data}. Given matrix $\A\in\R^{d\times d}$, define prediction
    \[
    \hat y_{\A}=\sgn{\x^\top\A\hat\bmu_s}.
    \]
    where $\hat\bmu_{s}$ is the SPI estimator defined in \eqref{pred spi}. Let $\Ac^\st:=\min_{\A\in\R^{d\times d}}\P(\hat y_\A\neq y)$ be its optimal solution set. Then, $\bmu\bmu^\top\in\Ac^\st$. Additionally, it obeys
    \begin{align}\label{n infty err}
    \P(\hat y_{\bmu\bmu^\top}\neq y)=\underset{\text{Bayes~error}}{\underbrace{Q(1/\sigma)}}+Q(\sqrt{np}/{\sigma})-2Q({1}/{\sigma})Q({\sqrt{np}}/{\sigma}).
    \end{align}
\end{theorem}
%
%
Note that, $\P(\hat y_{\bmu\bmu^\top}\neq y)$ depends on $np$ and $\sigma$ only, regardless of $\bmu$ and $d$. 
%
%
%
\begin{theorem}\label{thm n infty}
    Let the prompt $\Z$ be generated as described in Section~\ref{sec icl}, and consider an $L$-layer linear attention model with $L \geq 2$ and $n = \infty$.  Additionally, let $\hat\bmu_{s}$ be the SPI estimator defined in \eqref{pred spi}. There exist model constructions such that for any $\Z$ following \eqref{def Z}, its prediction satisfies
    \[
    y_{\att\mbar L}(\Z)=\sgn{\x^\top\bmu\bmu^
    \top\hat\bmu_s}.
    \]
\end{theorem}
The proof follows directly from Proposition~\ref{prop multilayer} ({label propagation}), which shows that multi-layer linear attention can output $\x^\top(\X^\top \X/n-\sigma^2\Iden)\hat\bmu_s$. As $n \to \infty$, the empirical covariance converges to its expectation, i.e., $\X^\top \X / n - \sigma^2 \Iden \to \bmu \bmu^\top$. The results in Figure~\ref{fig diff n 10000} validate Theorem~\ref{thm n infty}, showing that as $n$ becomes large enough (i.e., $n=10000$), the predictions from both 2-layer and 5-layer linear attention models, as well as the SSPI-$1$ and SSPI-$\infty$ estimators, closely align with the classification error characterized in Theorem~\ref{thm optimal A}, depicted by the black dotted line.

Theorem~\ref{thm n infty} establishes that, with infinitely many unlabeled samples, an $L$-layer linear attention model (for $L \geq 2$) can implement the predictor characterized in Theorem~\ref{thm optimal A} using the optimal choice of $\A$, thereby achieving the classification error specified in~\eqref{n infty err}. In the following, we shift to the non-asymptotic setting where $n$ is finite and analyze the model's performance in this regime.

\begin{theorem}\label{thm non asymp}
     Let the prompt $\Z$ be generated as described in Section~\ref{sec icl}. Consider an $L$-layer linear attention model with $L \geq 2$ and denote its optimal prediction as $y^\st_{\att\mbar L}(\Z)$. Additionally, let $\hat\bmu_{s}$ be the SPI estimator defined in \eqref{pred spi}. Suppose that the number of labeled samples satisfies $np\geq 8d\sigma^2$ and $n>O(d)$ is sufficiently large. Then, there exists a universal constant $C>0$ such that the classification error satisfies
    \[
    \P(y_{\att\mbar L}^\st(\Z)\neq y)\leq Q\left(\frac{1-C\sqrt{d/n}}{\sigma}\right)+ e^{-d}.
    \]
\end{theorem}
The proof is deferred to Appendix~\ref{app sec non asymp}. Note that when $n\gg d$, the classification error approaches the Bayes error, i.e., $\P(y_{\att\mbar L}^\st(\Z)\neq y)\approx Q(1/\sigma)$.

\subsection{Multi-layer Attention as Expectation Maximization and Belief Propagation}\label{em section}
In Section~\ref{sec multi poly}, we discussed how multi-layer linear attention can express polynomial functions of $\X^\top\X$. Here, we further explore the connection between multi-layer attention and the Expectation Maximization (EM) algorithm for semi-supervised learning. Beyond linear attention, we also highlight key differences between linear and softmax-based attention mechanisms, particularly in how they implement labeling strategies analogous to those in the EM algorithm.

Consider the following construction of the $\ell$-th layer attention weights:
\[
\W_q=\W_k=\begin{bmatrix}
    \Iden_d&0\\
    0&0
\end{bmatrix}\quad\text{and}\quad\W_v=\begin{bmatrix}
    0&0\\
    0&c_\ell
\end{bmatrix}.
\]
We examine both linear and softmax attention mechanisms. Let $\sft{\cdot}$ denotes the softmax operation that applies on the rows of a matrix. With this, the data update defined in~\eqref{def input ell} becomes:
\begin{align*}
    &\text{Linear attention:}~~~\quad\y_{\ell+1}=\y_\ell+c_\ell\X\X^\top\y_\ell\\
    &\text{Softmax attention:}\quad\y_{\ell+1}=\y_\ell+c_\ell\sft{\X\X^\top}\y_\ell
\end{align*}
In the case of linear attention, given the pseudo-labels $\y_\ell = [y_1^\ell, y_2^\ell, \ldots, y_n^\ell]$ at layer $\ell$, the model estimates the task mean using the SPI algorithm (cf.~\eqref{pred spi}) as $\hat{\bmu}_\ell = \X^\top \y_\ell$. The attention then updates each pseudo-label through the residual rule:
$$y_i^\ell\to y_i^\ell+c_\ell\x_i^\top \hat\bmu_{\ell},$$
where $c_\ell$ is a layer-specific coefficient. Owing to the linearity of this mechanism, the resulting pseudo-labeling strategy aligns with a linear EM-style update.

In contrast, softmax attention computes pairwise similarities via the softmax of dot products. Define $s_{ij}=\frac{e^{\x_i^\top\x_j}}{\sum_{j\leq n}e^{\x_i^\top\x_j}}$, then each pseudo-label is updated via: 
$$y_i^\ell\to y_i^\ell+c_\ell\sum_{j\leq n}s_{ij}y_{j}^\ell.$$
This update is a nonlinear, similarity-weighted pseudo-labeling strategy which can also be viewed as \emph{belief propagation}. The nonlinear nature highlights a key distinction between softmax and linear attention in how they emulate EM–like updates ($s_{ij}=\x_i^\top\x_j$ for linear attention).

\section{Experiments}\label{sec exp}




\begin{figure}[!t]
\centering
\begin{subfigure}{0.325\linewidth}    
    \centering
    \includegraphics[width=\linewidth]{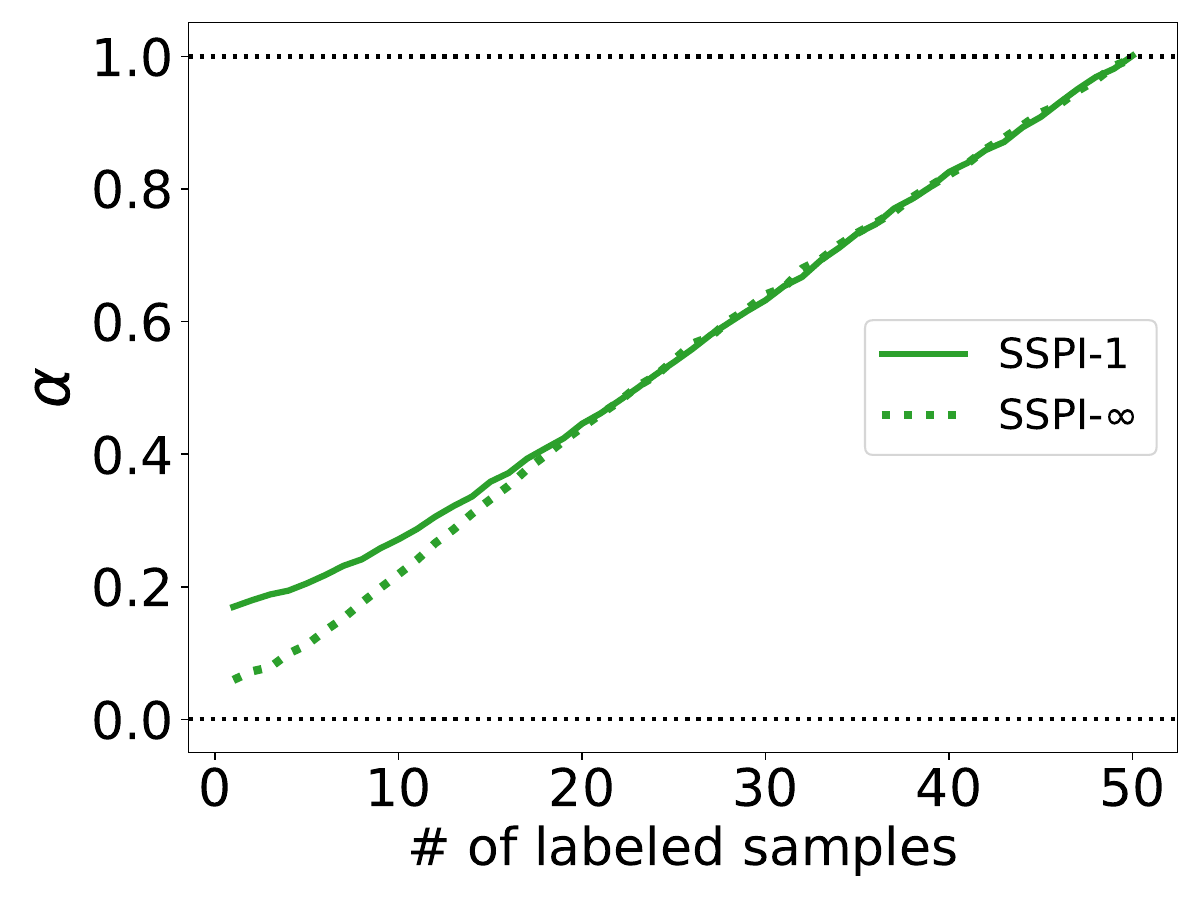}
    \caption{$n=50$, $p \in (0, 1]$.}\label{fig diff m alpha}
\end{subfigure}
\begin{subfigure}{0.325\linewidth}    
    \centering
    \includegraphics[width=\linewidth]{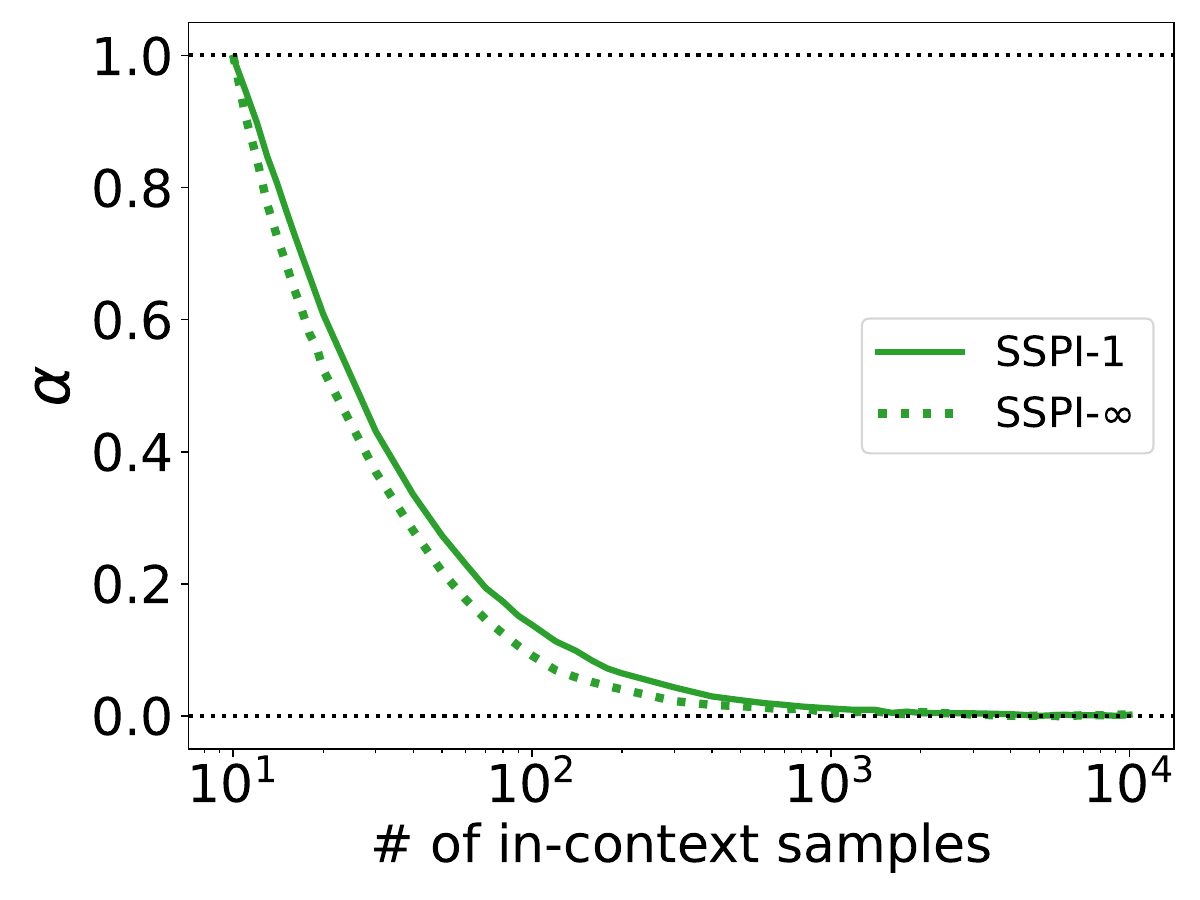}
\caption{$np=10$, $n \in \{10,\cdots,10000\}$.}\label{fig diff n 10000 alpha}
\end{subfigure}
\begin{subfigure}{0.325\linewidth}    
    \centering
    \includegraphics[width=\linewidth]{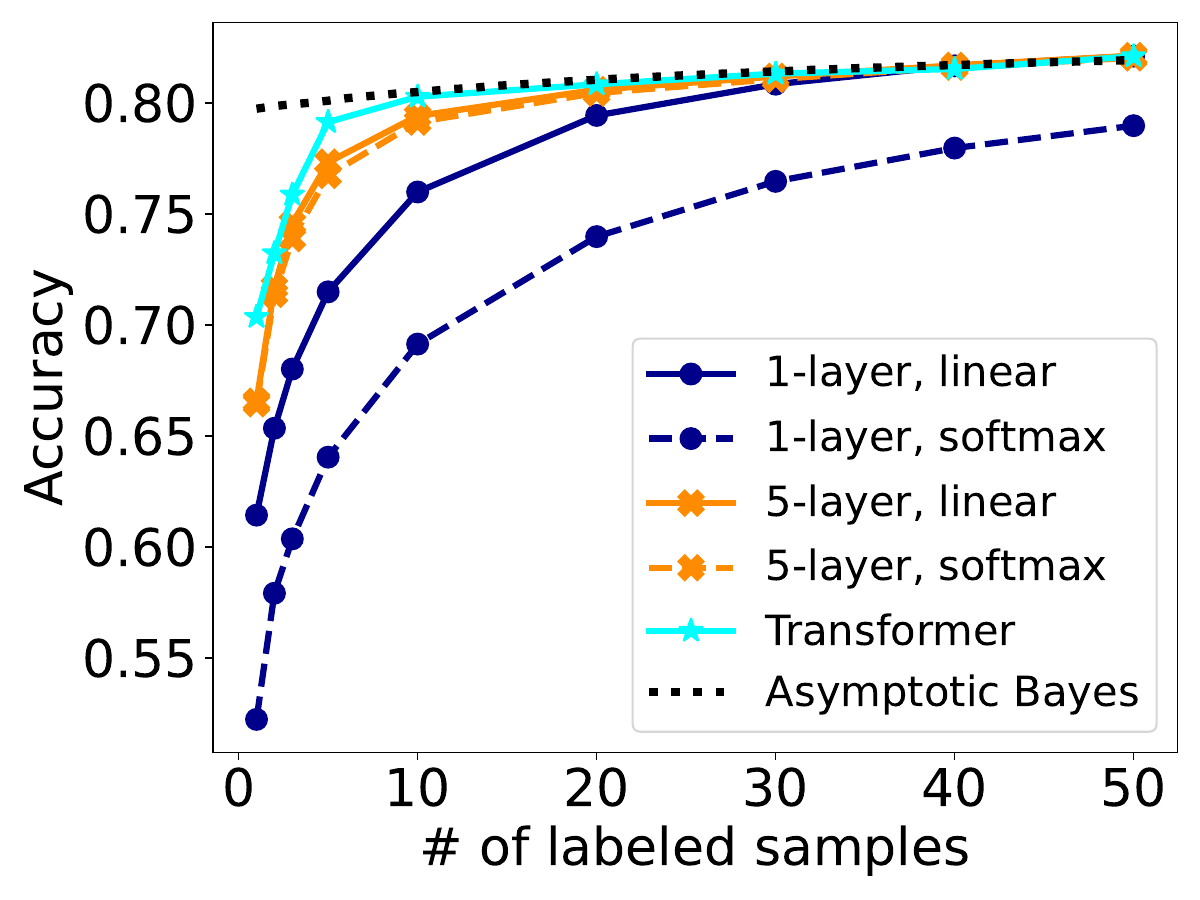}
\caption{$n=50$, $p \in (0, 1]$.}\label{fig tf}
\end{subfigure}
\caption{\small{Additional experimental results. (a)\&(b): Analysis of the optimal $\alpha$ values for the SSPI estimator (cf.~\eqref{pred sspi}) under varying $(n, p, k)$. Green solid and dotted curves represent optimal $\alpha$ values for SSPI-$1$ and SSPI-$\infty$, respectively. The SSPI results shown in Figure \ref{fig multi layer} use the corresponding $\alpha$ values from Figs. \ref{fig diff m alpha} and \ref{fig diff n 10000 alpha}. (c): Comparison of different model architectures for the SS-ICL problem. Dark blue and orange curves show results for 1-layer and 5-layer attention models, with solid and dashed curves representing linear and softmax attention, respectively. Cyan curves correspond to 5-layer Transformers. The black dotted curve shows the asymptotic Bayes-optimal error (cf.~\cite{lelarge2019asymptotic}). Results suggest the performance ordering: Transformer > linear attention > softmax attention. Further details are provided in Section~\ref{sec exp}.}}\label{fig alpha}
\end{figure}


In Sections~\ref{sec one layer} and~\ref{sec multilayer}, we introduced Figure~\ref{fig multi layer} and demonstrated its consistency with our theoretical results. In this section, we describe the experimental setup and implementation details. Additionally, we present further empirical findings to investigate additional questions of interest in Section~\ref{sec add exp}. Motivated by Proposition~\ref{prop looped}, which suggests that looping can help leverage unlabeled data, Section~\ref{sec tabpfn} introduces an algorithm based on the TabPFN, showing how it can enhance prediction performance by incorporating a small amount of unlabeled data and iterative pseudo-labeling through model looping.

\paragraph*{Experimental Setup}
Following Section~\ref{sec setup}, set $d=10$ and noise level $\sigma=1$. All models are trained using Adam optimizer with a learning rate of $10^{-3}$ for 40,000 epochs, with a batch size of 512. We use logistic loss in our experiments. Since our study focuses on the optimization landscape and model expressivity, and experiments are implemented via gradient descent, we repeat 10 trainings from random initialization and results are presented as the maximal test accuracy among those 10 trails.

\subsection{Additional Observations}\label{sec add exp}
\paragraph*{{Exploration of Optimal $\alpha$ Values}}
In Section~\ref{sec multilayer}, we introduced the SSPI-$k$ estimator (cf.~\eqref{pred sspi}), but did not discuss the choice of the mixing parameter $\alpha$, which plays a crucial role in balancing the contribution of the supervised estimator $\hat{\bmu}_s$. Specifically, $\alpha$ controls how much weight is given to the purely supervised signal. In the fully supervised case, the optimal choice is $\alpha = 1$, as $\hat{\bmu}_s$ corresponds to the optimal estimator.

In Figures~\ref{fig diff m alpha} and~\ref{fig diff n 10000 alpha}, we empirically examine the optimal values of $\alpha$. Given $\bmu \sim \text{Unif}(\mathbb{S}^{d-1})$, we define the optimal $\alpha$ as the minimizer of the following cosine similarity-based objective:
\[
\alpha^\st:=\min_{\alpha\in[0,1]}\Lc(\alpha)\quad\text{where}\quad\Lc(\alpha)=1-\E[\texttt{cosine\_similarity}(\bmu_{ss\mbar k},\bmu)].
\]
For each setting, we optimize $\alpha$ using the Adam optimizer for 10,000 epochs with a batch size of 128 and a learning rate of 0.01. The results are shown in Figs~\ref{fig diff m alpha} and~\ref{fig diff n 10000 alpha}.

In Figure~\ref{fig diff m alpha}, for both SSPI-$1$ and SSPI-$\infty$, the optimal $\alpha$ starts near zero when the number of labeled examples is small, reflecting the limited utility of $\hat{\bmu}_s$ in low-supervision regimes. As the number of labeled samples increases, $\alpha$ grows approximately linearly and approaches $1$ when the problem becomes fully supervised.
In Figure~\ref{fig diff n 10000 alpha}, when $n = 10$ and $p = 1$ (i.e., all examples are labeled), the optimal $\alpha$ begins at $1$. As $n$ increases and the fraction of unlabeled data grows, $\alpha$ decreases significantly. This trend indicates that as the volume of unlabeled data increases, the SSPI estimator adaptively reduces reliance on the supervised component $\hat{\bmu}_s$ and increases reliance on the semi-supervised component, which leverages the structure of the unlabeled data through $\X^\top \X$.

\paragraph*{{Comparison Across Different Model Architectures}}
Beyond linear attention, we investigate additional model architectures under our SS-ICL setting. The comparison results are presented in Fig.~\ref{fig tf}. The softmax attention model uses the same structure described in Section~\ref{sec icl}, with the only difference being the addition of a softmax operation in Eq.~\eqref{def att}. The Transformer model introduces further nonlinearity and capacity by incorporating multi-layer perceptrons (MLPs) and layer normalization. The Transformer experiments are conducted with 5-layer models.

When comparing weaker models—such as 1-layer linear (dark blue solid) and softmax (dark blue dashed) attention—we observe that softmax attention consistently underperforms linear attention. Notably, softmax attention fails to match the performance of the optimal supervised estimator, even when all labels are observed (i.e., when the number of labeled samples equals $n=50$). Furthermore, increasing the depth of softmax attention (orange dashed curve for 5-layer softmax) still does not surpass the performance of 5-layer linear attention (orange solid curve).
Among all architectures, the Transformer achieves the best performance due to its increased model capacity and expressiveness. Compared with Fig.~\ref{fig diff m}, where the orange and dark blue markers (linear attention) are identical, the Transformer significantly improves accuracy. This improvement highlights that SSPI, while effective, is not the optimal semi-supervised estimator.
Although our semi-supervised setting assumes isotropic data, the characterization of its optimal algorithm remains an open and foundational problem for future exploration.
In the figure, we also include the asymptotic Bayes-optimal curve (black dotted; derived from \cite{lelarge2019asymptotic}) . As the number of samples increases, the results from linear attention, softmax attention, and Transformer all converge toward this optimal curve. We attribute the initial performance gap, particularly at low values along $x$-axis (e.g., $np = 1$), to the scarcity of labeled data.

\begin{algorithm}[t]
\caption{\textbf{LoopTabFM}: {Looping Tabular FM with Soft Pseudo-labels and Risk-aware Updates}}
\begin{algorithmic}[1]
\Require Dataset $\Dc_{\text{lab}},\Dc_{\text{unlab}}$, looping iterations $K$
\Procedure{\texttt{Looping}}{$\Dc_{\text{lab}},\Dc_{\text{unlab}},K$}
  \State \texttt{FM}$_{0}$ $\gets\text{TabPFN-v2}(\Dc_{\text{lab}})$\Comment{\texttt{FM}$_k$ corresponds to model of Loop-$k$.}
  \State $\Dc_{\text{unlab}}\gets\texttt{FM}_0(\Dc_{\text{unlab}})$ \Comment{Assign pseudo labels via $\hat y^\text{soft}\gets \texttt{FM}_{0}$$(\x\in\Dc_\text{unlab})$.}
    \State $\texttt{FM}_{\text{best}}\gets \texttt{FM}_{0}$
    \State $\Rc_{\text{val}}=\texttt{Val\_Risk}(\Dc_{\text{unlab}})$
    \For{Looping iteration $k=1,\dots,K$}
    \State $\texttt{FM}_k\gets\text{TabPFN-v2}(\Dc_{\text{lab}}\cup\Dc_{\text{unlab}})$
    \State $\Dc_{\text{unlab}}\gets\texttt{FM}_k(\Dc_{\text{unlab}})$ \Comment{Update pseudo labels via $\hat y^\text{soft}\gets \texttt{FM}_{k}$$(\x\in\Dc_\text{unlab})$.}

    \If{$\texttt{Val\_Risk}(\Dc_{\text{unlab}})<\Rc_{\text{val}}$}
      \State $\texttt{FM}_{\text{best}}\gets\texttt{FM}_k$
      \State $\Rc_{\text{val}}=\texttt{Val\_Risk}(\Dc_{\text{unlab}})$
    \EndIf
  \EndFor
  \State \Return \texttt{FM}$_\text{best}$
\EndProcedure

\Procedure{\texttt{Val\_Risk}}{$\Dc_{\text{unlab}}$}
  \State \Return $\frac{1}{|\Dc_{\text{unlab}}|}\sum_i\min(|\hat y^{\text{soft}}_i-1|,\;|\hat y_i^{\text{soft}}+1|)$
  \State\Comment{$\hat y^{\text{soft}}$ corresponds to the assigned soft label for feature in $\Dc_{\text{unlab}}$.}
\EndProcedure
\end{algorithmic}
\label{alg:looping-tabpfn}
\end{algorithm}

\subsection{Tabular Experiments}\label{sec tabpfn}

\begin{table}[t]
    \centering
    \scriptsize
    \setlength\tabcolsep{3pt}
    \caption{\small{Comparison of test accuracy (\%) between the baseline (Loop-0) and LoopTabFM (Algorithm \ref{alg:looping-tabpfn}) after 1 to 5 iterations using TabPFN-v2. Each result is averaged over 100 random trials. The highest test accuracy for each dataset is highlighted in bold. The final column reports the relative improvement (\%) of Loop-5 over the baseline, computed as (Loop-5 $-$ Loop-0)$/$Loop-0$\times$100\%. Positive signs indicate a performance improvement over the baseline, while negative signs indicate a performance drop.}}
    \label{tab:tabpfn}
    
    \begin{tabular}{l|lll |c cc cc c|r}
        \toprule
        \textbf{OpenML ID} 
        & \textbf{\# of features} 
        & \textbf{\# of samples} 
        & \textbf{Class imbalance}
        & \textbf{Loop-0} 
        & \textbf{Loop-1} 
        & \textbf{Loop-2} 
        & \textbf{Loop-3} 
        & \textbf{Loop-4} 
        & \textbf{Loop-5}&\textbf{Rel. Imp. (\%)}\\
        \midrule
        3&36&3196&1.09&58.62&	58.63&	58.45&	58.69&	\textbf{59.00}&	58.97&0.60 $(+)$\\
        31&20&1000&2.33&\textbf{66.18}&	65.95&	66.05&	65.58&	65.52&	65.07&1.68 $(-)$\\
        1049&37&1458&7.19&72.00&	75.62&	79.48&	80.31&	\textbf{81.49}&	81.40&13.06 $(+)$
\\
        1067&21&2109&5.47&73.12&	76.59&	77.94&	77.92&	78.57&	\textbf{78.60}&7.50 $(+)$
\\
        1464&4&748&3.20&60.46&	63.96&	70.20&	71.29&	\textbf{72.26}&	72.18&19.38 $(+)$
\\
        1487&72&2534&14.84&82.54&	87.67&	88.57&	88.27&	\textbf{89.85}&	89.56&8.51 $(+)$
\\
        1489&5&5404&2.41&66.40&	67.62&	\textbf{68.30}&	68.14&	68.21&	68.18&2.69 $(+)$
\\
        1494&41&1055&1.96&62.24&	63.05&	64.62&	65.94&	\textbf{66.07}&	66.05&6.12 $(+)$
\\
        40701&20&5000&6.07&66.45&	70.65&	75.99&	\textbf{78.18}&	78.00&	77.70&16.93 $(+)$
\\
        40900&36&5100&67&\textbf{98.53}&	98.41&	98.39&	98.39&	98.27&	98.26&0.28 $(-)$
\\
        40981&14&690&1.25&73.56&	74.41&	74.67&	\textbf{74.99}&	74.93&	74.94&1.88 $(+)$
\\
        40983&5&4839&17.54&79.71&	85.04&	89.36&	\textbf{92.94}&	92.90&	92.75&16.35 $(+)$
\\
        41143&144&2984&1&64.64&	64.80&	65.06&	65.17&	\textbf{65.29}&	65.13&0.76 $(+)$
\\
        41144&259&3140&1.01&50.70&	50.63&	50.68&	50.67&	50.71&	\textbf{50.77}&0.14 $(+)$
\\
        41145&308&5832&1&56.16&	\textbf{56.28}&	56.21&	56.24&	56.19&	56.22&0.12 $(+)$
\\
        41146&20&5124&1&71.26&	73.90&	75.39&	75.84&	76.02&	\textbf{77.07}&8.51 $(+)$
\\
        41156&48&4147&3.03&67.74&	69.78&	70.64&	\textbf{71.82}&	71.72&	71.74&5.90 $(+)$
\\
        \midrule
        \textbf{Average}&&&&68.84&70.76&72.35&72.96&\textbf{73.24}&73.21&6.35 $(+)$\\
        \bottomrule
    \end{tabular}
\end{table}

To further investigate how model looping (Proposition~\ref{prop looped}) can improve label prediction, we propose the LoopTabFM algorithm that addresses unlabeled data by iteratively assigning pseudo-labels, with its details outlined in Algorithm~\ref{alg:looping-tabpfn}.
Suppose that we are given labeled $\Dc_{\text{lab}}$ and unlabeled $\Dc_{\text{unlab}}$ datasets. The overall workflow of the algorithm proceeds as follows: 
\begin{enumerate}
    \item \textbf{Base Model:} Perform ICL using TabPFN on the labeled dataset $\Dc_{\text{lab}}$ and treat the resulting model as the base model (Loop-0). The corresponding test accuracies are reported in Table~\ref{tab:tabpfn}.
    \item \textbf{Pseudo-Label Assignment:} Using the current model (e.g., Loop-\( k \)) to generate predictions for the unlabeled data \( \Dc_{\text{unlab}} \). Assign soft pseudo-labels based on these predictions. Note that the model outputs are scalars (i.e., elements of \( \mathbb{R} \)) and can be interpreted as soft labels.
    \item \textbf{Model Update:} Construct a new prompt by combining the labeled examples with their true labels and the unlabeled examples with their assigned soft pseudo-labels. Perform ICL using TabPFN on this combined prompt to obtain an updated model (Loop-$(k+1)$). Repeat this process from Step 2 until the maximum number of looping iterations is reached. 
    \item[$\st$] \textbf{Model Validation:} To improve the stability of the looping process, we introduce an additional validation step and retain the model with the lowest validation risk as the final (best) model. 
    Specifically, after assigning soft pseudo-labels to the unlabeled data, i.e., $\Dc_{\text{unlab}}=\{(\x_i,\hat y_i^{\text{soft}})_{i=1}^n\}$, we compute the validation risk over these pseudo-labeled examples as follows:
\[
\texttt{Val\_Risk}(\Dc_{\text{unlab}})=\frac{1}{n}\sum_{i\in[n]}\min\left(\left|\hat y^{\text{soft}}_i-1\right|,\;\left|\hat y_i^{\text{soft}}+1\right|\right),
\]
which penalizes predictions that deviate from confident binary labels $\pm1$.
\end{enumerate}

%
We evaluated the effectiveness of our proposed looping strategy by iteratively applying TabPFN-v2 on real-world binary classification benchmarks used in \cite{hollmann2025accurate}. 
%
The results are summarized in Table~\ref{tab:tabpfn}, where each entry represents an average over 100 random splits of the dataset, with 80\% of the data used as the test set in each split. 

For each experiment, we randomly sample 10 labeled and 10 unlabeled examples, ensuring that the labeled set includes at least one example from each class. As a baseline (Loop-0), we apply TabPFN-v2 using only the labeled data. The corresponding test accuracies are reported in the ``Loop-0'' column of Table~\ref{tab:tabpfn}. We compare this to models updated through up to $k \leq 5$ iterations of pseudo-label update, with results shown in the ``Loop-$k$'' columns. The final column reports the relative improvement (Rel. Imp.) over the baseline. Our results demonstrate that the looping strategy can significantly improve test accuracy. For instance, on OpenML datasets 1049, 1464, 40701, and 40983, accuracy improves by more than 10\% over the baseline using only 10 additional unlabeled samples. 
The last row of the table reports average performance across datasets, revealing that the majority of performance gains occur in the first two iterations. This observation aligns with our synthetic experiments using multi-layer models (Figure~\ref{fig multi layer}), where the improvement from 1-layer to 2-layer is substantially greater than the improvement from 2-layer to 5-layer. 
These findings highlight that explicitly looping the tabular foundation model to iteratively refine soft pseudo-labels of unlabeled data using only a few iterations can substantially enhance performance.

{As shown and discussed, our LoopTabFM algorithm enhances model performance. However, this improvement is not consistent across all datasets. For example, performance drops on the OpenML datasets with IDs 31 and 40900. This may be attributed to factors such as noise levels in the raw data, class imbalance, or other dataset-specific characteristics. 
In contrast to our synthetic experimental setting, where the model is pretrained in a meta-learning fashion on the distribution of the given dataset, TabPFN is used as a general-purpose pretrained foundation model and applied directly to target datasets in a single-shot inference setting. Prior work \citep{ye2025closer} has also shown that TabPFN can be sensitive to input length, which may further affect performance consistency.
Despite these limitations, our experiments with TabPFN offer an initial insight into how unlabeled data and iterative looping can be leveraged to improve predictive performance. These findings suggest promising future directions, such as designing data-aware looping algorithms that adapt to dataset-specific properties.}

\section{Discussion and Limitations}

Our paper introduces a theoretical study of semisupervised in-context learning and characterizes how transformer, specifically linear attention, models can harness unlabeled data in their context window to make inference. We show that depth is crucial to go beyond supervised estimation and utilize unlabeled data, and the latter is achieved by constructing estimators of the form $\hat{\bmu}=\sum_{i= 0}^K a_i (\X^\top\X)^i\X^\top\y$. $\log K$ depth suffices to express a $K$th order polynomial which is in line with our synthetic and real experiments that corroborate that mild amount of depth/looping already achieves most of the benefit. Our core theoretical results are limited to linear attention models and it is important to understand the capabilities of the full transformer architecture. Indeed, transformer (MLP+softmax) empirically outperforms a linear attention model with equal number of layers, well approximating the Bayes optimal semisupervised estimator. It would also be exciting to go beyond the classification setting and examine how self-generated CoT rationales, as in \citep{wu2023many}, can enhance ICL capabilities for tasks that require reasoning/autoregression. {Additionally, our proposed LoopTabFM algorithm demonstrates that iteratively pseudo-labeling unlabeled data can indeed enhance predictive performance for tabular tasks. However, there remains significant potential for developing more intelligent, data-specific algorithms that more effectively leverage unlabeled data to further improve model performance.}

\section*{Acknowledgements}

{This work was supported in part by the National Science Foundation grants CCF-2046816, CCF-2403075, CCF-2008020, the Office of Naval Research grant N000142412289, and by gifts/awards from Open Philanthropy, Amazon Research, and Google Research.}

\bibliography{ref}

@article{liu2026unlabeled,
  title={Unlabeled Data Can Provably Enhance In-Context Learning of Transformers},
  author={Liu, Renpu and Yang, Jing},
  journal={arXiv preprint arXiv:2601.10058},
  year={2026}
}

@inproceedings{ma2025tabdpt,
  title={Tabdpt: Scaling tabular foundation models on real data},
  author={Ma, Junwei and Thomas, Valentin and Hosseinzadeh, Rasa and Labach, Alex and Cresswell, Jesse C and Golestan, Keyvan and Yu, Guangwei and Caterini, Anthony L and Volkovs, Maksims},
  booktitle={The Thirty-ninth Annual Conference on Neural Information Processing Systems},
  year={2025}
}

@article{ye2025closer,
  title   = {A closer look at tabpfn v2: Strength, limitation, and extension},
  author  = {Ye, Han-Jia and Liu, Si-Yang and Chao, Wei-Lun},
  journal = {arXiv preprint arXiv:2502.17361},
  year    = {2025}
}

@article{chen2025maple,
  title   = {MAPLE: Many-Shot Adaptive Pseudo-Labeling for In-Context Learning},
  author  = {Chen, Zihan and Wang, Song and Tan, Zhen and Li, Jundong and Shen, Cong},
  journal = {arXiv preprint arXiv:2505.16225},
  year    = {2025}
}

@article{cao2021risk,
  title   = {Risk bounds for over-parameterized maximum margin classification on sub-gaussian mixtures},
  author  = {Cao, Yuan and Gu, Quanquan and Belkin, Mikhail},
  journal = {Advances in Neural Information Processing Systems},
  volume  = {34},
  pages   = {8407--8418},
  year    = {2021}
}

@article{li2024fine,
  title   = {Fine-grained analysis of in-context linear estimation: Data, architecture, and beyond},
  author  = {Li, Yingcong and Rawat, Ankit S and Oymak, Samet},
  journal = {Advances in Neural Information Processing Systems},
  volume  = {37},
  pages   = {138324--138364},
  year    = {2024}
}

@article{montanari2019generalization,
  title   = {The generalization error of max-margin linear classifiers: High-dimensional asymptotics in the overparametrized regime},
  author  = {Montanari, Andrea and Ruan, Feng and Sohn, Youngtak and Yan, Jun},
  journal = {arXiv preprint arXiv:1911.01544},
  volume  = {7},
  year    = {2019}
}

@article{deng2022model,
  title     = {A model of double descent for high-dimensional binary linear classification},
  author    = {Deng, Zeyu and Kammoun, Abla and Thrampoulidis, Christos},
  journal   = {Information and Inference: A Journal of the IMA},
  volume    = {11},
  number    = {2},
  pages     = {435--495},
  year      = {2022},
  publisher = {Oxford University Press}
}

@article{chatterji2021finite,
  title   = {Finite-sample analysis of interpolating linear classifiers in the overparameterized regime},
  author  = {Chatterji, Niladri S and Long, Philip M},
  journal = {Journal of Machine Learning Research},
  volume  = {22},
  number  = {129},
  pages   = {1--30},
  year    = {2021}
}

@article{belkin2018overfitting,
  title   = {Overfitting or perfect fitting? risk bounds for classification and regression rules that interpolate},
  author  = {Belkin, Mikhail and Hsu, Daniel J and Mitra, Partha},
  journal = {Advances in Neural Information Processing Systems},
  volume  = {31},
  year    = {2018}
}

@article{bartlett2006convexity,
  title     = {Convexity, classification, and risk bounds},
  author    = {Bartlett, Peter L and Jordan, Michael I and McAuliffe, Jon D},
  journal   = {Journal of the American Statistical Association},
  volume    = {101},
  number    = {473},
  pages     = {138--156},
  year      = {2006},
  publisher = {Taylor \& Francis}
}

@book{hastie2009elements,
  title     = {The elements of statistical learning: data mining, inference, and prediction},
  author    = {Hastie, Trevor and Tibshirani, Robert and Friedman, Jerome H},
  volume    = {2},
  year      = {2009},
  publisher = {Springer}
}

@book{devroye2013probabilistic,
  title     = {A probabilistic theory of pattern recognition},
  author    = {Devroye, Luc and Gy{\"o}rfi, L{\'a}szl{\'o} and Lugosi, G{\'a}bor},
  volume    = {31},
  year      = {2013},
  publisher = {Springer Science \& Business Media}
}

@misc{neopane2019lecture,
  author       = {Ojash Neopane},
  title        = {Lecture Notes on High-Dimensional Statistics},
  year         = {2018},
  howpublished = {\url{https://www.stat.cmu.edu/~arinaldo/Teaching/36709/S19/Scribed_Lectures/Feb26_Ojash.pdf}}
}

@article{li2025gating,
  title   = {Gating is Weighting: Understanding Gated Linear Attention through In-context Learning},
  author  = {Li, Yingcong and Tarzanagh, Davoud Ataee and Rawat, Ankit Singh and Fazel, Maryam and Oymak, Samet},
  journal = {arXiv preprint arXiv:2504.04308},
  year    = {2025}
}

@article{laurent2000adaptive,
  title     = {Adaptive estimation of a quadratic functional by model selection},
  author    = {Laurent, Beatrice and Massart, Pascal},
  journal   = {Annals of statistics},
  pages     = {1302--1338},
  year      = {2000},
  publisher = {JSTOR}
}

@article{bai2023transformers,
  title   = {Transformers as statisticians: Provable in-context learning with in-context algorithm selection},
  author  = {Bai, Yu and Chen, Fan and Wang, Huan and Xiong, Caiming and Mei, Song},
  journal = {Advances in neural information processing systems},
  volume  = {36},
  pages   = {57125--57211},
  year    = {2023}
}

@article{yang2024context,
  title   = {In-context learning with representations: Contextual generalization of trained transformers},
  author  = {Yang, Tong and Huang, Yu and Liang, Yingbin and Chi, Yuejie},
  journal = {arXiv preprint arXiv:2408.10147},
  year    = {2024}
}

@inproceedings{oymak2020statistical,
  title        = {A theoretical characterization of semi-supervised learning with self-training for gaussian mixture models},
  author       = {Oymak, Samet and Gulcu, Talha Cihad},
  booktitle    = {International Conference on Artificial Intelligence and Statistics},
  pages        = {3601--3609},
  year         = {2021},
  organization = {PMLR}
}

@article{hollmann2022tabpfn,
  title   = {Tabpfn: A transformer that solves small tabular classification problems in a second},
  author  = {Hollmann, Noah and M{\"u}ller, Samuel and Eggensperger, Katharina and Hutter, Frank},
  journal = {arXiv preprint arXiv:2207.01848},
  year    = {2022}
}

@inproceedings{mahankali2024one,
  title     = {One Step of Gradient Descent is Provably the Optimal In-Context Learner with One Layer of Linear Self-Attention},
  author    = {Arvind V. Mahankali and Tatsunori Hashimoto and Tengyu Ma},
  booktitle = {The Twelfth International Conference on Learning Representations},
  year      = {2024},
  url       = {https://openreview.net/forum?id=8p3fu56lKc}
}

@article{wu2023many,
  title   = {How Many Pretraining Tasks Are Needed for In-Context Learning of Linear Regression?},
  author  = {Wu, Jingfeng and Zou, Difan and Chen, Zixiang and Braverman, Vladimir and Gu, Quanquan and Bartlett, Peter L},
  journal = {arXiv preprint arXiv:2310.08391},
  year    = {2023}
}

@inproceedings{kwon2020algorithm,
  title        = {The EM algorithm gives sample-optimality for learning mixtures of well-separated gaussians},
  author       = {Kwon, Jeongyeol and Caramanis, Constantine},
  booktitle    = {Conference on Learning Theory},
  pages        = {2425--2487},
  year         = {2020},
  organization = {PMLR}
}

@article{wei2020theoretical,
  title   = {Theoretical analysis of self-training with deep networks on unlabeled data},
  author  = {Wei, Colin and Shen, Kendrick and Chen, Yining and Ma, Tengyu},
  journal = {arXiv preprint arXiv:2010.03622},
  year    = {2020}
}

@article{hollmann2025accurate,
  title     = {Accurate predictions on small data with a tabular foundation model},
  author    = {Hollmann, Noah and M{\"u}ller, Samuel and Purucker, Lennart and Krishnakumar, Arjun and K{\"o}rfer, Max and Hoo, Shi Bin and Schirrmeister, Robin Tibor and Hutter, Frank},
  journal   = {Nature},
  volume    = {637},
  number    = {8045},
  pages     = {319--326},
  year      = {2025},
  publisher = {Nature Publishing Group UK London}
}

@article{agarwal2024many,
  title   = {Many-Shot In-Context Learning},
  author  = {Agarwal, Rishabh and Singh, Avi and Zhang, Lei M and Bohnet, Bernd and Chan, Stephanie and Anand, Ankesh and Abbas, Zaheer and Nova, Azade and Co-Reyes, John D and Chu, Eric and others},
  journal = {arXiv preprint arXiv:2404.11018},
  year    = {2024}
}

@inproceedings{li2023transformers,
  title        = {Transformers as algorithms: Generalization and stability in in-context learning},
  author       = {Li, Yingcong and Ildiz, Muhammed Emrullah and Papailiopoulos, Dimitris and Oymak, Samet},
  booktitle    = {International Conference on Machine Learning},
  pages        = {19565--19594},
  year         = {2023},
  organization = {PMLR}
}

@inproceedings{dai2023gpt,
  title     = {Why Can {GPT} Learn In-Context? Language Models Secretly Perform Gradient Descent as Meta-Optimizers},
  author    = {Dai, Damai  and
               Sun, Yutao  and
               Dong, Li  and
               Hao, Yaru  and
               Ma, Shuming  and
               Sui, Zhifang  and
               Wei, Furu},
  editor    = {Rogers, Anna  and
               Boyd-Graber, Jordan  and
               Okazaki, Naoaki},
  booktitle = {Findings of the Association for Computational Linguistics: ACL 2023},
  month     = jul,
  year      = {2023},
  address   = {Toronto, Canada},
  publisher = {Association for Computational Linguistics},
  url       = {https://aclanthology.org/2023.findings-acl.247},
  doi       = {10.18653/v1/2023.findings-acl.247},
  pages     = {4005--4019}
}

@inproceedings{xie2022an,
  title     = {An Explanation of In-context Learning as Implicit Bayesian Inference},
  author    = {Sang Michael Xie and Aditi Raghunathan and Percy Liang and Tengyu Ma},
  booktitle = {International Conference on Learning Representations},
  year      = {2022},
  url       = {https://openreview.net/forum?id=RdJVFCHjUMI}
}

@article{ahn2024transformers,
  title   = {Transformers learn to implement preconditioned gradient descent for in-context learning},
  author  = {Ahn, Kwangjun and Cheng, Xiang and Daneshmand, Hadi and Sra, Suvrit},
  journal = {Advances in Neural Information Processing Systems},
  volume  = {36},
  year    = {2023}
}

@article{garg2022can,
  title   = {What can transformers learn in-context? a case study of simple function classes},
  author  = {Garg, Shivam and Tsipras, Dimitris and Liang, Percy S and Valiant, Gregory},
  journal = {Advances in Neural Information Processing Systems},
  volume  = {35},
  pages   = {30583--30598},
  year    = {2022}
}

@inproceedings{akyrek2023what,
  title     = {What learning algorithm is in-context learning? Investigations with linear models},
  author    = {Ekin Aky{\"u}rek and Dale Schuurmans and Jacob Andreas and Tengyu Ma and Denny Zhou},
  booktitle = {The Eleventh International Conference on Learning Representations },
  year      = {2023},
  url       = {https://openreview.net/forum?id=0g0X4H8yN4I}
}

@article{collins2024context,
  title   = {In-Context Learning with Transformers: Softmax Attention Adapts to Function Lipschitzness},
  author  = {Collins, Liam and Parulekar, Advait and Mokhtari, Aryan and Sanghavi, Sujay and Shakkottai, Sanjay},
  journal = {arXiv preprint arXiv:2402.11639},
  year    = {2024}
}

@inproceedings{wangtheoretical,
  title     = {A theoretical understanding of self-correction through in-context alignment},
  author    = {Wang, Yifei and Wu, Yuyang and Wei, Zeming and Jegelka, Stefanie and Wang, Yisen},
  booktitle = {Advances in Neural Information Processing Systems},
  volume    = {37},
  pages     = {89869--89912},
  year      = {2024}
}

@article{qu2025tabicl,
  title   = {TabICL: A Tabular Foundation Model for In-Context Learning on Large Data},
  author  = {Qu, Jingang and Holzm{\"u}ller, David and Varoquaux, Ga{\"e}l and Morvan, Marine Le},
  journal = {arXiv preprint arXiv:2502.05564},
  year    = {2025}
}

@article{shen2024training,
  title   = {On the training convergence of transformers for in-context classification},
  author  = {Shen, Wei and Zhou, Ruida and Yang, Jing and Shen, Cong},
  journal = {arXiv preprint arXiv:2410.11778},
  year    = {2024}
}

@inproceedings{von2023transformers,
  title        = {Transformers learn in-context by gradient descent},
  author       = {Von Oswald, Johannes and Niklasson, Eyvind and Randazzo, Ettore and Sacramento, Jo{\~a}o and Mordvintsev, Alexander and Zhmoginov, Andrey and Vladymyrov, Max},
  booktitle    = {International Conference on Machine Learning},
  pages        = {35151--35174},
  year         = {2023},
  organization = {PMLR}
}

@article{guo2025deepseek,
  title   = {Deepseek-r1: Incentivizing reasoning capability in llms via reinforcement learning},
  author  = {Guo, Daya and Yang, Dejian and Zhang, Haowei and Song, Junxiao and Zhang, Ruoyu and Xu, Runxin and Zhu, Qihao and Ma, Shirong and Wang, Peiyi and Bi, Xiao and others},
  journal = {arXiv preprint arXiv:2501.12948},
  year    = {2025}
}

@article{min2022rethinking,
  title   = {Rethinking the role of demonstrations: What makes in-context learning work?},
  author  = {Min, Sewon and Lyu, Xinxi and Holtzman, Ari and Artetxe, Mikel and Lewis, Mike and Hajishirzi, Hannaneh and Zettlemoyer, Luke},
  journal = {arXiv preprint arXiv:2202.12837},
  year    = {2022}
}

@inproceedings{zhou2024visual,
  title     = {Visual In-Context Learning for Large Vision-Language Models},
  author    = {Zhou, Yucheng and Li, Xiang and Wang, Qianning and Shen, Jianbing},
  booktitle = {Findings of the Association for Computational Linguistics ACL 2024},
  pages     = {15890--15902},
  year      = {2024}
}

@article{snell2024scaling,
  title   = {Scaling llm test-time compute optimally can be more effective than scaling model parameters},
  author  = {Snell, Charlie and Lee, Jaehoon and Xu, Kelvin and Kumar, Aviral},
  journal = {arXiv preprint arXiv:2408.03314},
  year    = {2024}
}

@article{brown2020language,
  title   = {Language models are few-shot learners},
  author  = {Brown, Tom and Mann, Benjamin and Ryder, Nick and Subbiah, Melanie and Kaplan, Jared D and Dhariwal, Prafulla and Neelakantan, Arvind and Shyam, Pranav and Sastry, Girish and Askell, Amanda and others},
  journal = {Advances in neural information processing systems},
  volume  = {33},
  pages   = {1877--1901},
  year    = {2020}
}

@article{zhang2024trained,
  title   = {Trained Transformers Learn Linear Models In-Context},
  author  = {Zhang, Ruiqi and Frei, Spencer and Bartlett, Peter L},
  journal = {Journal of Machine Learning Research},
  volume  = {25},
  number  = {49},
  pages   = {1--55},
  year    = {2024}
}

@inproceedings{lelarge2019asymptotic,
  title        = {Asymptotic bayes risk for gaussian mixture in a semi-supervised setting},
  author       = {Lelarge, Marc and Miolane, L{\'e}o},
  booktitle    = {2019 IEEE 8th International Workshop on Computational Advances in Multi-Sensor Adaptive Processing (CAMSAP)},
  pages        = {639--643},
  year         = {2019},
  organization = {IEEE}
}

@article{krishnapuram2004semi,
  title   = {On semi-supervised classification},
  author  = {Krishnapuram, Balaji and Williams, David and Xue, Ya and Carin, Lawrence and Figueiredo, M{\'a}rio and Hartemink, Alexander},
  journal = {Advances in neural information processing systems},
  volume  = {17},
  year    = {2004}
}

@article{thrampoulidis2020theoretical,
  title   = {Theoretical insights into multiclass classification: A high-dimensional asymptotic view},
  author  = {Thrampoulidis, Christos and Oymak, Samet and Soltanolkotabi, Mahdi},
  journal = {Advances in Neural Information Processing Systems},
  volume  = {33},
  pages   = {8907--8920},
  year    = {2020}
}

@article{wang2022binary,
  title     = {Binary classification of gaussian mixtures: Abundance of support vectors, benign overfitting, and regularization},
  author    = {Wang, Ke and Thrampoulidis, Christos},
  journal   = {SIAM Journal on Mathematics of Data Science},
  volume    = {4},
  number    = {1},
  pages     = {260--284},
  year      = {2022},
  publisher = {SIAM}
}

@inproceedings{ratsaby1995learning,
  title        = {Learning from a Mixture of Labeled and Unlabeled Examples with Parametric Side Information},
  author       = {Ratsaby, Joel and Venkatesh, Santosh S.},
  booktitle    = {Proceedings of the Eighth Annual Conference on Computational Learning Theory (COLT ’95)},
  pages        = {412--417},
  year         = {1995},
  organization = {ACM},
  doi          = {10.1145/225298.225348}
}

@article{nigam2000text,
  title   = {Text Classification from Labeled and Unlabeled Documents using EM},
  author  = {Nigam, Kamal and McCallum, Andrew Kachites and Thrun, Sebastian and Mitchell, Tom},
  journal = {Machine Learning},
  volume  = {39},
  number  = {2--3},
  pages   = {103--134},
  year    = {2000},
  doi     = {10.1023/A:1007692713085}
}

@article{balakrishnan2017statistical,
  title   = {Statistical guarantees for the EM algorithm: From population to sample-based analysis},
  author  = {Balakrishnan, Sivaraman and Wainwright, Martin J and Yu, Bin},
  journal = {The Annals of Statistics},
  volume  = {45},
  number  = {1},
  pages   = {77--120},
  year    = {2017}
}

@article{xu2024expectation,
  title    = {Expectation maximisation pseudo labels},
  journal  = {Medical Image Analysis},
  volume   = {94},
  pages    = {103125},
  year     = {2024},
  issn     = {1361-8415},
  doi      = {https://doi.org/10.1016/j.media.2024.103125},
  url      = {https://www.sciencedirect.com/science/article/pii/S1361841524000501},
  author   = {Moucheng Xu and Yukun Zhou and Chen Jin and Marius {de Groot} and Daniel C. Alexander and Neil P. Oxtoby and Yipeng Hu and Joseph Jacob}
}
\bibliographystyle{colm2025_conference}

\newpage
\addtocontents{toc}{\protect\setcounter{tocdepth}{3}}
\tableofcontents
\appendix
\section{Analysis of Single-layer Linear Attention}\label{app sec one layer}
\subsection{Supporting Lemmas}
Recap the SPI estimator from \eqref{pred spi}. Given a semi-supervised dataset $(\x_i,y_i)_{i=1}^n$ as described in Section~\ref{sec data}, let $\Ic$ denote the token indices set corresponding to the labeled demonstrations, that is, we have
\begin{align}\label{app def Ic}
y_{i}=\begin{cases}
    y_i^c,&i\in\Ic\\
    0,&otherwise.
\end{cases}
\end{align}
Then, the SPI estimates the task mean via 
\[
\hat\bmu_s=\frac{1}{|\Ic|}\sum_{i\in\Ic}y_i\x_i.
\]
Let $\W\in\R^{d\times d}$ be the preconditioning matrix. We define the following objective:
\begin{align}
    \W^\st:=\arg\min_{\W\in\R^{d\times d}}\Lct(\W)\quad\text{where}\quad\Lct(\W)=\E\left[\left(\x^\top\W\sum_{i\in\Ic} y_i\x_i-y\right)^2\right].\label{app W loss}
\end{align}
Here, we set $(\x,y)$ to be the query feature and its corresponding true label. The expectation subsumes the randomness in $(\x_i,y_i),(\x,y)$ as described in Section~\ref{sec data}. 

In the following, we provide a lemma that establishes equivalence between optimizing $\Lc_{\att\mbar1}(\Wc^{(1)},\hb)$ (cf.~\eqref{obj att} and choosing $L=1$) and $\Lct(\W)$.



\begin{lemma}\label{lemma reduce}
    Consider ICL problem described in Section~\ref{sec icl} with prompt defined in \eqref{def Z}.  Consider training a single-layer linear attention with squared loss, that is, $L=1$ and $\ell(y,\hat y)=(y-\hat y)^2$. Recall the objectives from \eqref{obj att} and \eqref{app W loss}, and let $\Lc^\st_{\att\text{-}1}$ and $\Lct^\st:=\Lct(\W^\st)$ be their corresponding optimal losses where $\W^\st$ is defined in \eqref{app W loss}. 
    %
    Then, we have 
    \begin{align}
    \Lc^\st_{\att\text{-}1}=\Lct^\st.\label{app equal loss}
    \end{align}
    Additionally, let $f_{\att\mbar1}^\st:\R^{(n+1)\times(d+1)}\to\R$ denote the optimal prediction (associated with the optimal loss $\Lc^\st_{\att\mbar1}$). We have that $f_{\att\mbar1}^\st$ is unique and for any prompt $\Z$ (cf.~\eqref{def Z})
    \begin{align}
        f_{\att\mbar1}^\st(\Z)=\x^\top\W^\st\sum_{i\in\Ic}y_i\x_i.\label{app equal pred}
    \end{align}
\end{lemma}
\begin{proof}
Recap the single-layer linear attention model and its prediction from \eqref{def att} and \eqref{def f att}. We have
\begin{align}
f_{\att\text{-}1}(\Z)=\hb^\top\att(\Z;\Wc)_{[n+1]}\quad\text{where}\quad\att(\Z;\Wc)=(\Z\W_q\W_k^\top\Z^\top)\M\Z\W_v\label{app def one layer att}
\end{align}
with $\Wc:=\{\W_q,\W_k,\W_v\}$ being the set of the query, key and value matrices of the attention. 
Since $\Wc$ and $\hb$ are tunable parameters, without loss of generality and for simplicity, let 
\[
\W:=\W_q\W_k^\top\quad\text{and}\quad\hbb:=\W_v\hb.
\]
Following the proof of \citealp[Proposition 1]{li2024fine}, similarly, we denote 
\[
\W=\begin{bmatrix}
    \Wb&\w_1\\
    \w_2^\top&w
\end{bmatrix}\qquad\text{and}\qquad\hbb=\begin{bmatrix}
    \hb_1\\
    h
\end{bmatrix},
\]
where $\Wb\in\R^{d\times d}$, $\w_1,\w_2,\hb_1\in\R^d$, and $w,h\in\R$.

Additionally, let $\Ic$ denote the token indices set corresponding to the labeled demonstrations (cf.~\eqref{app def Ic}). 
Recall the prompt $\Z$ from \eqref{def Z}, and $\X=[\x_1~\cdots~\x_n]^\top\in\R^{n\times d}$ and $\y=[y_1~\cdots~y_n]^\top\in\R^n$ from \eqref{def X y}. Then we get
\begin{align}\label{app def Z}
    \Z=\begin{bmatrix}
        \x_1&\x_2&\cdots&\x_n&\x\\
        y_1& y_2&\cdots& y_n& 0
    \end{bmatrix}^\top=\begin{bmatrix}
        \X^\top&\x\\
        \y^\top&0
    \end{bmatrix}^\top\in\R^{(n+1)\times (d+1)}.
\end{align}

Combining \eqref{app def one layer att} and \eqref{app def Z} together, we can rewrite the one-layer linear prediction as 
\begin{align*}
    f_{\att\mbar1}(\Z)&=[\x^\top~~0]\W\Z^\top\M\Z\hbb\\
    &=[\x^\top~~0]\begin{bmatrix}
    \Wb&\w_1\\
    \w_2^\top&w
\end{bmatrix}\begin{bmatrix}
        \X^\top&\x\\
        \y^\top&0
    \end{bmatrix}\begin{bmatrix}
        \Iden_n&0\\0&0
    \end{bmatrix}\begin{bmatrix}
        \X^\top&\x\\
        \y^\top&0
    \end{bmatrix}^\top\begin{bmatrix}
    \hb_1\\
    h
\end{bmatrix}\\
    &=[\x^\top\Wb~~\x^\top\w_1]\begin{bmatrix}
        \X^\top\X&\X^\top\y\\
        \y^\top\X&\y^\top\y
    \end{bmatrix}\begin{bmatrix}
    \hb_1\\
    h
\end{bmatrix}\\
&=[\x^\top\Wb~~\x^\top\w_1]\begin{bmatrix}
    \X^\top\X\hb_1+h\X^\top\y\\
    \y^\top\X\hb_1+h\y^\top\y
\end{bmatrix}\\
&=\x^\top\Wb(\X^\top\X\hb_1+h\X^\top\y)+\x^\top\w_1(\y^\top\X\hb_1+h\y^\top\y)\\
&=\x^\top(h\Wb+\w_1\hb_1^\top)\X^\top\y+\x^\top(\Wb\X^\top\X\hb_1+h\y^\top\y\w_1)\\
&=\x^\top\Wt\X^\top\y+\x^\top(\Wb\X^\top\X\hb_1+mh\w_1)
\end{align*}
where $\Wt:=h\Wb+\w_1\hb_1^\top$ and we define $m:=|\Ic|$. 


Next, recall the loss from \eqref{obj att} and consider the squared loss function, $\ell(y,\hat y)=(y-\hat y)^2$. We have
\begin{align*}
\Lc_{\att\mbar1}(\Wc^{(1)},\hb)&=
\E\left[(f_{\att\mbar1}(\Z)-y)^2\right]\\
&=\E\left[\left(\x^\top\Wt\X\y+\x^\top\left(\Wb\X^\top\X\hb_1+mh\w_1\right)-y\right)^2\right]\\
&=\E\left[\left(y\x^\top\Wt\X\y+y\x^\top\left(\Wb\X^\top\X\hb_1+mh\w_1\right)-1\right)^2\right].
\end{align*}

For simplicity and without loss of generality,  we omit $y$ and use $\x$ to represent $y\x$. Note that the distribution of (updated) $\x$ is not conditioned on its class and given mean vector $\bmu$, it follows  $\x\sim\Nc(\bmu,\sigma^2\Iden)$. Similarly, let $\x_i$ represent $y_i^c\x_i$. We can then write 
\begin{align}
\Lc_{\att\mbar1}(\Wc^{(1)},\hb)&=\E\left[\left(\x^\top\Wt\sum_{i\in\Ic}\x_i+\x^\top\left(\Wb\X^\top\X\hb_1+mh\w_1\right)-1\right)^2\right]\label{app loss 1}\\
&=\E\left[\left(\x^\top\Wt\sum_{i\in\Ic}\x_i-1\right)^2\right]+\E\left[\left(\x^\top\left(\Wb\X^\top\X\hb_1+mh\w_1\right)\right)^2\right]\nn\\
&\quad+2\E\left[\left(\x^\top\Wt\sum_{i\in\Ic}\x_i-1\right)\left(\x^\top\left(\Wb\X^\top\X\hb_1+mh\w_1\right)\right)\right].\nn
\end{align}

We start with showing that for any given parameters $\W\in\R^{(d+1)\times(d+1)},\hb\in\R^{d+1}$, the component $\E[(\x^\top\Wt\sum_{i\in\Ic}\x_i-1)(\x^\top(\Wb\X^\top\X\hb_1+mh\w_1))]=0$. 
To prove it, we first expand
\begin{align*}
    &(\x^\top\Wt\sum_{i\in\Ic}\x_i-1)(\x^\top(\Wb\X^\top\X\hb_1+mh\w_1))\\
    &=\underset{(a)}{\underbrace{(\x^\top\Wt\sum_{i\in\Ic}\x_i)(\x^\top\Wb\X^\top\X\hb_1)}}-\underset{(b)}{\underbrace{\x^\top\Wb\X^\top\X\hb_1}}+\underset{(c)}{\underbrace{(\x^\top\Wt\sum_{i\in\Ic}\x_i)(mh\x^\top \w_1)}}-\underset{(d)}{\underbrace{mh\x^\top \w_1}}.
\end{align*}
In the following, we consider the expectations of $(a),(b),(c),(d)$ sequentially, all of which take the value zero. First note that since $\bmu\sim\text{Unif}(\mathbb{S}^{d-1})$ and $(\bxi_i)_{i=1}^n,\bxi\sim\Nc(0,\sigma^2\Iden)$, the odd moments of $\bmu$, $\bxi$ and  $\bxi_i,i\in[n]$ are all zeros.

\begin{align*}
    \quad(a):\quad&\E\left[(\x^\top\Wt\sum_{i\in\Ic}\x_i)(\x^\top\Wb\X^\top\X\hb_1)\right]\\
    &=\E\left[(\bmu+\bxi)^\top\Wt\sum_{i\in\Ic}(\bmu+\bxi_i)(\bmu+\bxi)^\top\Wb\sum_{i\in[n]}(\bmu+\bxi_i)(\bmu+\bxi_i)^\top\hb_1\right]\hspace*{300pt}\\
    &=\sum_{i\in\Ic}\sum_{j\in[n]}\E\left[(\bmu+\bxi)^\top\Wt(\bmu+\bxi_i)(\bmu+\bxi)^\top\Wb(\bmu+\bxi_j)(\bmu+\bxi_j)^\top\hb_1\right]\\
    &=\sum_{i\in\Ic}\sum_{j\in[n]}\E\left[\bmu^\top\Wt\bmu\bmu^\top\Wb(\bmu\bmu^\top+\bxi_j\bxi_j^\top)\hb_1+\bxi^\top\Wt\bmu\bxi^\top\Wb(\bmu\bmu^\top+\bxi_j\bxi_j^\top)\hb_1\right]\\
    &=0,
\end{align*}
\begin{align*}
    \quad(b):\quad&\E\left[\x^\top\Wb\X^\top\X\hb_1\right]\\
    &=\E\left[(\bmu+\bxi)^\top\Wb\sum_{i\in[n]}(\bmu+\bxi_i)(\bmu+\bxi_i)^\top\hb_1\right]\hspace*{300pt}\\
        &=\E\left[\bmu^\top\Wb\sum_{i\in[n]}(\bmu\bmu^\top+\bxi_i\bxi_i^\top)\hb_1\right]\\
        &=0,
\end{align*}
\begin{align*}
    \quad(c):\quad&\E\left[(\x^\top\Wt\sum_{i\in\Ic}\x_i)(mh\x^\top \w_1)\right]\\
    &=mh\E\left[(\bmu+\bxi)^\top\Wt\sum_{i\in\Ic}(\bmu+\bxi_i)(\bmu+\bxi)^\top\w_1\right]\hspace*{300pt}\\
    &=mh\sum_{i\in\Ic}\E\left[(\bmu+\bxi)^\top\Wt\bmu(\bmu+\bxi)^\top\w_1\right]\hspace*{300pt}\\
&=mh\sum_{i\in\Ic}\E\left[\bmu^\top\Wt\bmu\bmu^\top\w_1+\bxi^\top\Wt\bmu\bxi^\top\w_1\right]\\
&=0,
\end{align*}
\begin{align*}
    \quad(d):\quad&\E\left[mh\x^\top \w_1\right]=0.\hspace*{500pt}\\
\end{align*}
Therefore, loss in \eqref{app loss 1} returns
\begin{align*}
\Lc_{\att\mbar1}(\Wc^{(1)},\hb)&=\underset{\Lct(\Wt)}{\underbrace{\E\left[\left(\x^\top\Wt\sum_{i\in\Ic}\x_i-1\right)^2\right]}}+\E\left[\left(\x^\top(\Wb\X^\top\X\hb_1+mh\w_1)\right)^2\right].
\end{align*}
Here, the first term $\E[(\x^\top\Wt\sum_{i\in\Ic}\x_i-1)^2]=\Lct(\Wt)$ where $\Lct(\Wt)$ is defined in \eqref{app W loss}.

Recall that $\Wt=h\Wb+\w_1\hb_1^\top$. Then for any $\Wt\in\R^{d\times d}$, setting $\hb_1=\w_1=\zerbb_d$ and $h=1$ returns
$\E\left[\left(\x^\top\left(\Wb\X^\top\X\hb_1+mh\w_1\right)\right)^2\right]=0$, and then
\[
\Lc_{\att\mbar1}(\Wc^{(1)},\hb)=\E\left[\left(\x^\top\Wb\sum_{i\in\Ic}\x_i-1\right)^2\right]
\]
Therefore, optimizing $\Lc_{\att\mbar1}(\Wc^{(1)},\hb)$ returns the same minima as optimizing $\Lct(\W)$, which completes the proof of \eqref{app equal loss}.
Note that optimal loss $\Lc_{\att\mbar1}^\st$ depends on the labeled data $i\in\Ic$ only.

Furthermore, since $\Lct(\W)$ is strongly convex (see \eqref{app W loss expand}), $\W^\st$ exists and is unique. Therefore, \eqref{app equal loss} and uniqueness of $\W^\st$ leads to the conclusion \eqref{app equal pred}.
\end{proof}

\begin{lemma}\label{lemma W=I}
Consider the objective defined in \eqref{app W loss} with semi-supervised data following Section~\ref{sec setup}. Then the optimal solution $\W^\st$ satisfies
\[
\W^\st=c\Iden
\]
for some $c>0$.
\end{lemma}
\begin{proof}
    Recap the Objective~\eqref{app W loss} and its optimal solution $\W^\st$. Let $\Ic$ be the index set corresponding the labeled in-context examples, and $|\Ic|=m$. Note that, $m$ is also a random variable, independent of $\x_i,y_i^c,\x,y$.

    
    As in the proof of Lemma~\ref{lemma reduce}, we use $\x$ to represent $y\x$ and $\x_i$ to represent $y_i^c\x_i$ for simplicity, where (updated) $\x_i,\x\sim\Nc(\bmu,\sigma^2\Iden)$. Letting $\bxi',\bxi,\bxi_i\sim\Nc(0,\sigma^2\Iden)$ be independent, we obtain
    \begin{align}
        \Lct(\W)&=\E\left[(\x^\top\W\sum_{i\in\Ic} \x_i-1)^2\right]\label{app W loss expand}\\
        &=\E\left[((\bmu+\bxi)^\top\W\sum_{i\in\Ic} (\bmu+\bxi_i)-1)^2\right]\nn\\
        &=\E\left[((\bmu+\bxi)^\top\W (m\bmu+\sqrt{m}\bxi')-1)^2\right]\nn\\
        &=\E\left[m^2(\bmu^\top\W\bmu)^2+m(\bmu^\top\W\bxi')^2+m^2(\bxi^\top\W\bmu)^2+m(\bxi^\top\W\bxi')^2+1\right]-2\E\left[m\bmu^\top\W\bmu\right]\nn\\
        &=\frac{\E[m^2]}{d(d+2)}(\tr{\W}^2+\tr{\W\W^\top}+\tr{\W^2})+\frac{\E[m+m^2]}{d}\sigma^2\tr{\W\W^\top}\nn\\
        &\quad+\E[m]\sigma^4\tr{\W\W^\top}+1-\frac{2\E[m]}{d}\tr{\W}.\nn
    \end{align}
    Differentiating it results in
    \begin{align*}
        \nabla_{\W}\Lct(\W)=\frac{2\E[m^2]}{d(d+2)}(\tr{\W}\Iden+\W+\W^\top)+\frac{2\E[m+m^2]\sigma^2}{d}\W+2\E[m]\sigma^4\W-\frac{2\E[m]}{d}\Iden.
    \end{align*}
    Setting $\nabla_{\W}\Lct(\W)=0$, we obtain the optimal $\W^\st$
    \[
    \W^\st=\frac{1}{(1+\sigma^2)\E[m^2]/\E[m]+\sigma^2+\sigma^4d}\Iden,
    \]
    which leads to the conclusion  that $\W^\st=c\Iden$, for  $c=\frac{1}{(1+\sigma^2)\E[m^2]/\E[m]+\sigma^2+\sigma^4d}>0$. It completes the proof.
\end{proof}
\subsection{Proof of Theorem~\ref{thm one layer}}\label{app thm one layer}
\begin{proof}
Note that \eqref{1 layer pred} can be easily proven using Lemmas~\ref{lemma reduce} and \ref{lemma W=I}. Then, we focus on proving \eqref{one layer err}. 


    
    Given that \eqref{1 layer pred} holds, we can rewrite its classification error as
    \begin{align}
    \P( {y}_{\att\text{-}1}^\st(\Z)\neq y)
    =\P(\sgn{\x^\top\hat\bmu_s}\neq y)=\P(\sgn{y\x^\top\hat\bmu_s}\neq 1)\label{app err}
    \end{align}
    where $\hat\bmu_s=\frac{1}{|\Ic|}\sum_{i\in\Ic}y_i\x_i$ defined in \eqref{pred spi} and $\Ic$ is the index set of labeled samples. Let $m=|\Ic|$. 
    
    Recall from Section~\ref{sec data} where $\x\sim\Nc(y\cdot\bmu,\sigma^2\Iden)$. We can rewrite
    \begin{align*}
    &y\x=\bmu+\sigma \g_1
    \quad\text{where}\quad\g_1\sim\Nc(0,\Iden).
    \end{align*}
    Then for any given $\bmu,\hat\bmu_s$, we get
    \begin{align}
    \P\left(\sgn{y\x^\top\hat\bmu_s}\neq 1~\big|~\bmu,\hat\bmu_s\right)&=\P\left(\left(\bmu+\sigma \g_1\right)^\top\hat\bmu_s<0~\big|~\bmu,\hat\bmu_s\right)\nn\\
    &=\P\left(\bmu^\top\hat\bmu_s<\sigma \g_1^\top\hat\bmu_s~\big|~\bmu,\hat\bmu_s\right)\nn\\
    &=Q\left(\frac{\bmu^\top\hat\bmu_s}{\sigma\tn{\hat\bmu_s}}\right).\label{app err given mu hat}
    \end{align}
    Here $Q$-function is the tail distribution function of the standard normal distribution. 
    

    Next, similarly, given that $\x_i\sim\Nc(y_i\cdot\bmu,\sigma^2\Iden)$ for $i\in\Ic$, we can rewrite
    \[
    \hat\bmu_s=\frac{1}{m}\sum_{i\in\Ic}y_i\x_i=\bmu+\frac{\sigma}{\sqrt{m}}\g_2\quad\text{where}\quad\g_2\sim\Nc(0,\Iden).
    \]
    Then combining \eqref{app err} and \eqref{app err given mu hat}, we have
    \begin{align*}
        \P( {y}_{\att\text{-}1}^\st(\Z)\neq y)&=\E_{\bmu,\g_2}\left[Q\left(\frac{\bmu^\top\hat\bmu_s}{\sigma\tn{\hat\bmu_s}}\right)\right]\\
        &=\E_{\bmu,\g_2}\left[Q\left(\frac{\bmu^\top(\bmu+\frac{\sigma}{\sqrt{m}}\g_2)}{\sigma\tn{\bmu+\frac{\sigma}{\sqrt{m}}\g_2}}\right)\right]\\
        &=\E_{\bmu,\g_2}\left[Q\left(\frac{1+\frac{\sigma}{\sqrt{m}}\bmu^\top\g_2}{\sigma\sqrt{1+2\frac{\sigma}{\sqrt{m}}\bmu^\top\g_2+\frac{\sigma^2}{m}\tn{\g_2}^2}}\right)\right].
    \end{align*}
    Note that for any $\bmu$ with $\tn{\bmu}=1$, we have $\bmu^\top\g_2\sim\Nc(0,1)$. 
    Therefore, we can write 
    \[
    \bmu^\top\g_2=g\quad\text{where}\quad g\sim\Nc(0,1),
    \]
    and let $\Ub\in\R^{d\times d}$ be a unitary matrix with first row being $\bmu$. We can write
    \[
    \tn{\g_2}^2=\tn{\Ub\g_2}^2=g^2+h\quad\text{where}\quad h\sim\Xc^2_{d-1}.
    \]
    Here, $\Xc^2_{d-1}$ denotes chi-squared distribution with $(d-1)$ degrees of freedom. Then, we get 
    \begin{align}
        \P( {y}_{\att\text{-}1}^\st(\Z)\neq y)
        &=\E_{g,h}\left[Q\left(\frac{1+\frac{\sigma}{\sqrt{m}}g}{\sigma\sqrt{1+2\frac{\sigma}{\sqrt{m}}g+\frac{\sigma^2}{m}(g^2+h)}}\right)\right]\nn\\
        &=\E_{g,h}\left[Q\left(\frac{1+\frac{\sigma}{\sqrt{m}}g}{\sigma\sqrt{(1+\frac{\sigma}{\sqrt{m}}g)^2+\frac{\sigma^2}{m}h}}\right)\right],\nn\\
        &=\E_{g,h}\left[Q\left(\frac{1+\eps_\sigma g}{\sigma\sqrt{(1+\eps_\sigma g)^2+\eps^2_\sigma h}}\right)\right],\nn
    \end{align}
    where $\eps_\sigma:=\sigma/\sqrt{m}$. It completes the proof of \eqref{one layer err}. 

Next, we derive an upper bound for $\P( {y}_{\att\text{-}1}^\st(\Z)\neq y)$. 
Let $c:=\eps_\sigma^{-1}$. Then we have 
\begin{align}
     \P( {y}_{\att\text{-}1}^\st(\Z)\neq y)&=\E_{g,h}\left[Q\left(\frac{c+ g}{\sigma\sqrt{(c+g)^2+h}}\right)\right]\nn\\
     &=\E_{g\geq-\frac{c}{2},h}\left[Q\left(\frac{c+ g}{\sigma\sqrt{(c+g)^2+h}}\right)\right]+\E_{g<-\frac{c}{2},h}\left[Q\left(\frac{c+ g}{\sigma\sqrt{(c+g)^2+h}}\right)\right]\nn\\
     &\leq\E_{g\geq-\frac{c}{2},h}\left[Q\left(\frac{c+ g}{\sigma\sqrt{(c+g)^2+h}}\right)\right]+Q(c/2)\nn\\
     &=\E_{g\geq-\frac{c}{2},h}\left[Q\left(\frac{1}{\sigma\sqrt{1+h/(c+g)^2}}\right)\right]+Q(c/2),\label{app err upper bound}
\end{align}
where the inequality comes from the fact that $\P(g\leq-c/2)=Q(c/2)$ and $Q(x)\leq 1$ for any $x\in\R$. Next, we have
\begin{align*}
    \frac{1}{\sqrt{1+h/(c+g)^2}}\geq1-\frac{1}{2}\frac{h}{(c+g)^2}\geq1-\frac{2h}{c^2}.
\end{align*}
Here the first inequality comes from that $\frac{1}{\sqrt{1+x}}\geq1-\frac{1}{2}x$ and the second utilizes that $g\geq -\frac{c}{2}$.

    
    Since $h\sim\Xc^2_{d-1}$, from the Laurent-Massart inequality \citep{laurent2000adaptive}, we have that
    \begin{align*}
        \P\left(h\geq d-1+2\sqrt{(d-1)t_1}+2t_1\right)\leq e^{-t_1}.
    \end{align*}
    Therefore, we have that with probability at least $1-e^{-t_1}$
    \[
    \frac{1}{\sqrt{1+h/(c+g)^2}}\geq1-\frac{2(d-1+2\sqrt{(d-1)t_1}+2t_1)}{c^2}.
    \]
    Setting $t_1=d$, we get with probability at least $1-e^{-d}$
    \[
    \frac{1}{\sqrt{1+h/(c+g)^2}}\geq1-\frac{10d}{c^2}.
    \]
    Combining the result with \eqref{app err upper bound}, since $Q(x)\leq 1$ for $x\in\R$ and $Q(x)\leq e^{-x^2/2}$ for $x>1$, we get that
    \begin{align*}
    \P( {y}_{\att\text{-}1}^\st(\Z)\neq y)&\leq e^{-d}+Q(c/2)+Q\left(\frac{1}{\sigma}\left(1-\frac{10d}{c^2}\right)\right)\\
    &\leq e^{-d}+e^{-1/8\eps_\sigma^2}+Q\left(\frac{1}{\sigma}\left(1-10d\eps_\sigma^2\right)\right).
    \end{align*}
    It completes the proof. 

\end{proof}
\section{Analysis of Multi-layer Linear Attention}
\subsection{Proof of Proposition~\ref{prop multilayer}}\label{app sec prop multi}

\begin{proof}
We consider the following model constructions for the attention matrices in the $\ell$th layer, $\ell\in[L]$ and the final linear prediction head: 
\begin{equation}\label{app multi cons}
\begin{split}
    \text{$\ell$th layer:}&\quad\W_{q\ell}\W_{k\ell}^\top=\begin{bmatrix}
    \Iden_d&0\\
    0&0
\end{bmatrix}\quad\text{and}\quad\W_{v\ell}=\begin{bmatrix}
    a_{\ell}\Iden_{d}&0\\
    0&b_{\ell}
\end{bmatrix};\\
\text{Prediction head:}&\quad\hb=\begin{bmatrix}
    \zerbb_d\\c
\end{bmatrix}.
\end{split}
\end{equation}

Suppose the input to $\ell$th layer is
\begin{align*}
    \Z_\ell=\begin{bmatrix}
    \X_\ell&\y_\ell\\
    \x_\ell^\top&y_\ell
\end{bmatrix}\in\R^{(n+1)\times(d+1)}\quad\text{where}\quad\Z_1=\Z=\begin{bmatrix}
    \X&\y\\
    \x^\top&0
\end{bmatrix}.
\end{align*}
Recapping the model construction from \eqref{app multi cons}, the $\ell$th layer output returns
\begin{align}
    \left(\Z_\ell\W_{q\ell}\W_{k\ell}^\top\Z_\ell^\top\M\right)\Z_\ell\W_{v\ell}&=\begin{bmatrix}
    \X_\ell&\y_\ell\\
    \x^\top_\ell&y_\ell
\end{bmatrix}\begin{bmatrix}
    \Iden_d&0\\0&0
\end{bmatrix}\begin{bmatrix}
    \X_\ell^\top&\x_\ell\\
    \y^\top_\ell&y_\ell
\end{bmatrix}\M\begin{bmatrix}
    \X_\ell&\y_\ell\\
    \x_\ell^\top&y_\ell
\end{bmatrix}\begin{bmatrix}
    a_{\ell}\Iden_d&0\\0&b_{\ell}
\end{bmatrix}\nn\\
    &=\begin{bmatrix}
    \X_\ell\X_\ell^\top&\X_\ell\x_\ell\\
    \x_\ell^\top\X_\ell^\top&\x_\ell^\top\x_\ell
\end{bmatrix}\begin{bmatrix}
    \Iden_n&0\\0&0
\end{bmatrix}\begin{bmatrix}
    a_{\ell}\X_{\ell}&b_\ell\y_\ell\\
    a_{\ell}\x_\ell^\top&b_\ell y_\ell
\end{bmatrix}\nn\\
&=\begin{bmatrix}
    a_\ell\X_\ell\X_\ell^\top\X_\ell&b_\ell\X_\ell\X_\ell^\top\y_\ell\\
    a_\ell\x_\ell^\top\X_\ell^\top\X_\ell&b_\ell\x_\ell^\top\X_\ell^\top\y_\ell
\end{bmatrix}.\label{app ell layer output}
\end{align}
Therefore, following \eqref{def input ell}, the input of $(\ell+1)$th layer is
\begin{align}
\Z_{\ell+1}&=\Z_\ell+\begin{bmatrix}
    a_\ell\X_\ell\X_\ell^\top\X_\ell&b_\ell\X_\ell\X_\ell^\top\y_\ell\\
    a_\ell\x_\ell^\top\X_\ell^\top\X_\ell&b_\ell\x_\ell^\top\X_\ell^\top\y_\ell
\end{bmatrix}\nn\\
&=\begin{bmatrix}
    \X_\ell+a_\ell\X_\ell\X_\ell^\top\X_\ell&\y_\ell+b_\ell\X_\ell\X_\ell^\top\y_\ell\\
    \x_\ell^\top+a_\ell\x_\ell^\top\X_\ell^\top\X_\ell&y_\ell+b_\ell\x_\ell^\top\X_\ell^\top\y_\ell
\end{bmatrix}\in\R^{(n+1)\times(d+1)}.\label{app ell+1 input}
\end{align}

\paragraph*{$\bullet$ Label propagation:} 
We first focus on deriving label propagation results. Suppose that we have
\[
a_\ell=0\quad\text{for}\quad \ell\in[L].
\]
Then following \eqref{app ell layer output}, the output of $\ell$'th layer takes the following form:
\begin{align}
    \left(\Z_\ell\W_{q\ell}\W_{k\ell}^\top\Z_\ell^\top\M\right)\Z_\ell\W_{v\ell}=\begin{bmatrix}
    0&b_\ell\X_\ell\X_\ell^\top\y_\ell\\
    0&b_\ell\x_\ell^\top\X_\ell^\top\y_\ell
\end{bmatrix}.\nn
\end{align}
Here, the first $d$ coordinates of each token's output are zeros, and therefore, the corresponding input coordinates remain unchanged, and we have 
\[
\X_\ell\equiv\X\quad\text{and}\quad\x_\ell\equiv\x\quad\text{for}\quad \ell\in[L].
\]
The prediction (based on the last token output and after applying prediction head) is given by
\begin{align}\label{app label prop f}
f_{\all\mbar L}(\Z)=cb_L\x^\top\X^\top\y_L.
\end{align}
We next focus on obtaining $\y_L$. From \eqref{app ell+1 input}, we have
\begin{align*}
    \y_{\ell+1}=\y_\ell+b_\ell\X\X^\top\y_\ell=(\Iden+b_\ell\X\X^\top)\y_\ell.
\end{align*}
Therefore,
\begin{align*}
    \y_{L}&=\prod_{\ell=1}^{L-1}(\Iden+b_\ell\X\X^\top)\y.
\end{align*}
Combining with \eqref{app label prop f} results in
\[
f_{\all\mbar L}(\Z)=cb_L\x^\top\X^\top\prod_{\ell=1}^{L-1}(\Iden+b_\ell\X\X^\top)\y=cb_L\x^\top\prod_{\ell=1}^{L-1}(\Iden+b_\ell\X^\top\X)\X^\top\y.
\]
It completes the proof.

\paragraph*{$\bullet$ Feature propagation:} We now focus on the feature propagation setting. In contrast to the label propagation, let us assume that
\begin{align*}
&a_\ell\to\infty\quad\text{and}\quad b_\ell\to0^+\quad\text{for}\quad\ell\in[L].
\end{align*}
The prediction (following \eqref{app ell layer output}, based on the last token output and after applying prediction head) is given by
\begin{align}\label{app feature prop f}
f_{\all\mbar L}(\Z)=cb_L\x_L^\top\X_L^\top\y_L.
\end{align}
We first obtain $\y_L$. From \eqref{app ell+1 input} (since $b_{\ell}\to 0$), we have
\[
\y_{\ell+1}=\y_\ell+b_\ell\X\X^\top\y_\ell=\y_\ell.
\]
Therefore, 
\begin{align*}
    \y_{\ell}\equiv\y\quad\text{for}\quad\ell\in[L].
\end{align*}
Next, we focus on $\X_L,\x_L$. From \eqref{app ell+1 input}, as $a_\ell\to\infty$, we have
\begin{align*}
    &\X_{\ell+1}=\X_\ell+a_\ell\X_\ell\X_\ell^\top\X_\ell=\X_\ell(\Iden+a_\ell\X_\ell^\top\X_\ell)= a_\ell\X_\ell\X_\ell^\top\X_\ell;\\
    &\x^\top_{\ell+1}=\x^\top_\ell+a_\ell\x_\ell^\top\X_\ell^\top\X_\ell=\x_\ell^\top(\Iden+a_\ell\X_\ell^\top\X_\ell)= a_\ell\x_\ell^\top\X_\ell^\top\X_\ell.
\end{align*}
Therefore,
\begin{align*}
    \X_{L}&=a_{L-1}\X_{L-1}(\X_{L-1}^\top\X_{L-1})\\
    &=a_{L-1}a_{L-2}^3\X_{L-2}(\X_{L-2}^\top\X_{L-2})^{\frac{3^2-1}{2}}\\
    &=a_{L-1}a_{L-2}^3a_{L-3}^{3^2}\X_{L-3}(\X_{L-3}^\top\X_{L-3})^{\frac{3^3-1}{2}}\\
    &=\cdots\\
    &=a_{L-1}a_{L-2}^3a_{L-3}^{3^2}...a_1^{3^{L-2}}\X(\X^\top\X)^{\frac{3^{L-1}-1}{2}},
\end{align*}
and 
\begin{align*}
    \x_{L}^\top&=a_{L-1}\x_{L-1}^\top(\X_{L-1}^\top\X_{L-1})\\
    &=a_{L-1}a_{L-2}^3\x_{L-2}^\top(\X_{L-2}^\top\X_{L-2})^{\frac{3^2-1}{2}}\\
    &=a_{L-1}a_{L-2}^3a_{L-3}^{3^2}\x_{L-3}^\top(\X_{L-3}^\top\X_{L-3})^{\frac{3^3-1}{2}}\\
    &=\cdots\\
    &=a_{L-1}a_{L-2}^3a_{L-3}^{3^2}...a_1^{3^{L-2}}\x^\top(\X^\top\X)^{\frac{3^{L-1}-1}{2}}.
\end{align*}
Combining all together with \eqref{app feature prop f}, we have that
\begin{align*}
    f_{\all\mbar L}(\Z)&=cb_L\x_L^\top\X_L^\top\y_L\\
    &=cb_L\left(\prod_{\ell=1}^{L-1}a_{\ell}^{3^{L-1-\ell}}\right)^2\x^\top(\X^\top\X)^{3^{L-1}-1}\X^\top\y.
\end{align*}
It completes the proof.
\end{proof}
\subsection{Proof of Proposition~\ref{prop looped}}\label{app sec prop looped}
\begin{proof}
    The proof follows directly by adopting the same model construction and proof strategy as in Proposition~\ref{prop multilayer}, under the additional assumption that
    \[
    a_\ell=a\quad\text{and}\quad b_\ell=b\quad\text{for}\quad\ell\in[L].
    \]
\end{proof}

\subsection{Proof of Lemma~\ref{lemma label+feature}}
\begin{proof}
    In the proof of Proposition~\ref{prop multilayer}, we showed how to derive the label and feature propagation results by restricting the construction to either $a_\ell \equiv 0$ (for label propagation) or $(a_\ell \to \infty,b_\ell \to 0)$ (for feature propagation). Here, we consider a propagation process without imposing restrictions on the choices of $(a_\ell, b_\ell)$, and study the form of the final prediction returned by the model.

    To avoid the notation conflict, we express the matrix $\A$ in \eqref{A highest degree} as
    \[
    \A=\sum_{k=0}^{K}e_k(\X^\top\X)^k
    \]
    and let $\eb=[e_0~e_2~\cdots~e_{(3^L-3)/2}]^\top\in\R^{K+1}$.

    Recall the same model construction used in the proof of Proposition~\ref{prop multilayer}, defined in \eqref{app multi cons}. 
    From \eqref{app ell layer output}, we have that
    \begin{align*}
        f_{\att\mbar L}(\Z)&=cb_L\x_L^\top\X_L^\top\y_L
    \end{align*}
    where following \eqref{app ell+1 input}, we have
    \begin{align*}
        \X_{\ell+1}&=\X_\ell(\Iden+a_\ell\X_\ell^\top\X_\ell),\\
        \x_{\ell+1}^\top&=\x_\ell^\top(\Iden+a_\ell\X_\ell^\top\X_\ell),\\
        \y_{\ell+1}&=(\Iden+b_\ell\X_\ell\X_\ell^\top)\y_{\ell}.
    \end{align*}
    At each layer, the operations performed are linear combinations and multiplications involving $\X_\ell^\top\X_\ell$ and identity matrices scaled by the parameters $(a_\ell,b_\ell)$. Thus, each coefficient $e_k$ of $(\X^\top\X)^k$
  depends smoothly on the scalar parameters $(a_\ell,b_\ell)$.

    From \eqref{app ell layer output} and \eqref{app ell+1 input}, we have that
    \begin{align}
        f_{\att\mbar L}(\Z)&=cb_L\x_L^\top\X_L^\top\y_L\label{app L layer f}\\
        &=cb_L\cdot\x_{L-1}^\top(\Iden+a_{L-1}\X_{L-1}^\top\X_{L-1})^2(\Iden+b_{L-1}\X_{L-1}^\top\X_{L-1})\X_{L-1}^\top\y_{L-1}\nn\\
        &=\cdots\nn
    \end{align}
    That is, in the final $f_{\att\mbar L}(\Z)$ expression, the coefficients corresponding to different degrees of $(\X^\top \X)^k$ depend on the model parameters $cb_L$ and $(a_\ell, b_\ell)_{\ell=1}^{L-1}$, which together have at most $2L - 1$ degrees of freedom.
    Let $\cb=[cb_L~a_1~\cdots~a_{L-1}~b_1~\cdots~b_{L-1}]^\top$. 
    This means there exists a smooth function $g:\R^{2L-1}\to\R^{K}$ such that: $\eb=g(\cb)$.

    It remains to show that an $L$-layer linear attention model can produce terms involving powers of $\X^\top \X$ up to degree $(3^L - 3)/2$. 

    Let $f(\Z)$ be a function that contains terms of the form $\x^\top (\X^\top \X)^k \X^\top \y$ for various powers $k$. Define $\Pc(f(\Z))$ as the projection that extracts the highest degree $k$ present in $f(\Z)$. For example, $\Pc\big(\x^\top(\Iden + (\X^\top \X)^2)\X^\top \y\big) = 2$. Then from \eqref{app L layer f}, we have
    \begin{align*}
        \Pc(f_{\att\mbar L}(\Z))&=\Pc(\x_L^\top\X_L^\top\y_L)\\
        &=\Pc(\x_{L-1}^\top(\X_{L-1}^\top\X_{L-1})^3\X^\top_{L-1}\y_{L-1})\\
        &=\Pc(\x_{L-2}^\top(\X_{L-2}^\top\X_{L-2})(\X_{L-2}^\top\X_{L-2})^{3^2}(\X_{L-2}^\top\X_{L-2})^2\X^\top_{L-2}\y_{L-2})\\
        &=\Pc(\x_{L-2}^\top(\X_{L-2}^\top\X_{L-2})^{3^2+3}\X^\top_{L-2}\y_{L-2})\\
        &=\Pc(\x_{L-3}^\top(\X_{L-3}^\top\X_{L-3})(\X_{L-3}^\top\X_{L-3})^{3^3+3^2}(\X_{L-3}^\top\X_{L-3})^2\X^\top_{L-3}\y_{L-3})\\
        &=\Pc(\x_{L-3}^\top(\X_{L-3}^\top\X_{L-3})^{3^3+3^2+3}\X^\top_{L-3}\y_{L-3})\\
        &=\dots\\
        &=\Pc(\x^\top(\X^\top\X)^{3^{L-1}+\dots+3^2+3}\X^\top\y)\\
        &=3^{L-1}+\dots+3^2+3=\frac{3^L-3}{2}.
    \end{align*}
    It completes the proof.
 
\end{proof}

\subsection{Proof of Theorem~\ref{thm optimal A}}
\begin{proof}
    Let $\bxi\sim\Nc(0,\Iden)$ and rewrite $y\x=\bmu+\sigma\bxi$. For any matrix $\A\in\R^{d\times d}$, the prediction error of $\hat y_{\A}=\sgn{\x^\top\A\hat\bmu_s}$ given $\hat\bmu_s$ returns
    \begin{align}
        \P(\hat\y_{\A}\neq y~\big|~\hat\bmu_s)&=\P(y\x^\top\A\hat\bmu_s<0~\big|~\hat\bmu_s)\nn\\
        &=\P((\bmu+\sigma\bxi)^\top\A\hat\bmu_s<0~\big|~\hat\bmu_s)\nn\\
        &=Q\left(\frac{\bmu^\top\A\hat\bmu_s}{\sigma\tn{\A\hat\bmu_s}}\right).\label{app optim A conditioned err}
    \end{align}
    For any $\A\in\R^{d\times d}$, we can decompose it as 
    \[
    \A=\sum_{i=1}^{d}\lambda_i\ub_i\vb_i^\top
    \]
    where $\ub_1=\bmu$, $\tn{\ub_i}=1$ and $\ub_i^\top\ub_j=0$ for any $i\neq j$. Let $\lambda_1>0$.
    Then, we get
    \begin{align}
    \bmu^\top\A\hat\bmu_s&=\bmu^\top(\sum_{i=1}^d\lambda_i\ub_i\vb_i^\top)\hat\bmu_s\nn\\
        &=\sum_{i=1}^d\lambda_i\bmu^\top\ub_i\vb_i^\top\hat\bmu_s\nn\\
        &=\lambda_1\bmu^\top\ub_1\vb_1^\top\hat\bmu_s\nn\\
        &=\lambda_1\vb_1^\top\hat\bmu_s.\label{app optim A top}
    \end{align}
    Now consider $\tn{\A\hat\bmu_s}$ where we have
    \begin{align*}
        \A\hat\bmu_s&=\sum_{i=1}^d\lambda_i\ub_i\vb_i^\top\hat\bmu_s\\
        &=\lambda_1\bmu\vb_1^\top\hat\bmu_s+\sum_{i=2}^d\lambda_i\ub_i\vb_i^\top\hat\bmu_s.
    \end{align*}
    Since $\ub_i$, $i\neq1$ is orthogonal to $\bmu$, $\lambda_1\bmu\vb_1^\top\hat\bmu_s$ is orthogonal to $\sum_{i=2}^d\lambda_i\ub_i\vb_i^\top\hat\bmu_s$. Therefore, given $\tn{\ub_i}=1$ for all $i\in[d]$, it obeys
    \begin{align}
        \tn{\A\hat\bmu_s}^2=\tn{\lambda_1\bmu\vb_1^\top\hat\bmu_s}^2+\sum_{i=2}^d\tn{\lambda_i\ub_i\vb_i^\top\hat\bmu_s}^2=(\lambda_1\vb_1^\top\hat\bmu_s)^2+\lambda_1^2\sum_{i=2}^d(\lambda_1^{-1}\lambda_i\vb_i^\top\hat\bmu_s)^2.\label{app optim A bottom}
    \end{align}
    For simplicity, define 
    \[
    \Delta(\hat\bmu_s)=\sum_{i=2}^d(\lambda_1^{-1}\lambda_i\vb_i^\top\hat\bmu_s)^2
    \]
    where $\Delta(\cdot)$ is a function of $\lambda_1$ and $(\lambda_i,\vb_i)$'s for $i\geq2$, and we have
    \[
    \Delta(\hat\bmu_s)\geq0\quad\text{and}\quad\Delta(-\hat\bmu_s)=\Delta(\hat\bmu_s).
    \]
    Recall that $\hat\bmu_s$ is the SPI estimator (cf.~\eqref{pred spi}). Let $|\Ic|=m$. We can write $\hat\bmu_s=\bmu+\bxi'/\sqrt{m}$ where $\bxi'\sim\Nc(0,\sigma^2\Iden)$. 
    
    
    Using \eqref{app optim A conditioned err}, \eqref{app optim A top} and \eqref{app optim A bottom}, the classification error becomes
    \begin{align*}
        \P(\hat\y_\A\neq y)&=\E_{\hat\bmu_s}\left[Q\left(\frac{\bmu^\top\A\hat\bmu_s}{\sigma\tn{\A\hat\bmu_s}}\right)\right]\\
        &=\E_{\hat\bmu_s}\left[Q\left(\frac{\vb_1^\top\hat\bmu_s}{\sigma\sqrt{(\vb_1^\top\hat\bmu_s)^2+\Delta(\hat\bmu_s)}}\right)\right]\\
        &=\E_{\vb_1^\top\hat\bmu_s<0}\left[Q\left(\frac{\vb_1^\top\hat\bmu_s}{\sigma\sqrt{(\vb_1^\top\hat\bmu_s)^2+\Delta(\hat\bmu_s)}}\right)\right]+\E_{\vb_1^\top\hat\bmu_s\geq 0}\left[Q\left(\frac{\vb_1^\top\hat\bmu_s}{\sigma\sqrt{(\vb_1^\top\hat\bmu_s)^2+\Delta(\hat\bmu_s)}}\right)\right].
    \end{align*}
    First, note that for any $x>0$, $Q(x)<0.5<Q(-x)$. Therefore, 
    the optimal choice of $\vb_1 \in \mathbb{R}^d$ that minimizes $\mathbb{P}(\hat{y}_\A \neq y)$ is contained within the set of $\vb_1$ values that maximize $\mathbb{P}(\vb_1^\top \hat{\bmu}_s > 0)$. Let $\vb_1^\st:=\arg\max_{\vb_1\in\R^d}\P(\vb_1^\top\hat\bmu_s>0)$. Given that $\hat\bmu_s\sim\Nc(\bmu,\sigma^2/m\Iden)$, we have that $\vb_1^\st=c\bmu$ for $c>0$. Let $c=1$ and therefore, $\vb_1^\st=\bmu$ without loss of generality (since $\lambda_1$ can be any positive scalar). 
    Then we obtain 
    \begin{align*}
        \min_{\A\in\R^{d\times d}}\P(\hat\y_\A\neq y)
        &=\min_{\Delta}\E_{\hat\bmu_s}\left[Q\left(\frac{\bmu^\top\hat\bmu_s}{\sigma\sqrt{(\bmu^\top\hat\bmu_s)^2+\Delta(\hat\bmu_s)}}\right)\right].
    \end{align*}
    Let $f(\hat\bmu_s)$ be the probability density function of $\hat\bmu_s$. Since $\hat\bmu_s\sim\Nc(\bmu,\sigma^2/m\Iden)$, then it satisfies
    \begin{align}
    f(\hat\bmu_s)\geq f(-\hat\bmu_s)\quad\text{for any }\bmu^\top\hat\bmu_s>0.\label{app optim A pdf}
    \end{align}
    Therefore, the classification error becomes
    \begin{align*}
        \P(\hat\y_\A\neq y~\big|~\vb_1=\bmu)
        &=\int_{\hat\bmu_s}f(\hat\bmu_s)Q\left(\frac{\bmu^\top\hat\bmu_s}{\sigma\sqrt{(\bmu^\top\hat\bmu_s)^2+\Delta(\hat\bmu_s)}}\right)d\hat\bmu_s\\
        &=\int_{\bmu^\top\hat\bmu_s>0}f(\hat\bmu_s)Q\left(\frac{\bmu^\top\hat\bmu_s}{\sigma\sqrt{(\bmu^\top\hat\bmu_s)^2+\Delta(\hat\bmu_s)}}\right)+f(-\hat\bmu_s)Q\left(\frac{-\bmu^\top\hat\bmu_s}{\sigma\sqrt{(\bmu^\top\hat\bmu_s)^2+\Delta(\hat\bmu_s)}}\right)d\hat\bmu_s\\
        &=\int_{\bmu^\top\hat\bmu_s>0}\left(f(\hat\bmu_s)-f(-\hat\bmu_s)\right)Q\left(\frac{\bmu^\top\hat\bmu_s}{\sigma\sqrt{(\bmu^\top\hat\bmu_s)^2+\Delta(\hat\bmu_s)}}\right)+f(-\hat\bmu_s)d\hat\bmu_s.
    \end{align*}
    Following \eqref{app optim A pdf}, to minimize the error, we need minimize $Q\left(\frac{\bmu^\top\hat\bmu_s}{\sigma\sqrt{(\bmu^\top\hat\bmu_s)^2+\Delta(\hat\bmu_s)}}\right)$ for $\bmu^\top\hat\bmu_s>0$, which can be easily done by choosing $\lambda_i=0$ for $i\geq2$. Then we get $\Delta(\hat\bmu_s)\equiv0$. Therefore, the optimal solution set $\Ac^\st$ defined in Theorem~\ref{thm optimal A} satisfies:
    \[
    \Ac^\st=\left\{\lambda_1\bmu\bmu^\top~\big|~\lambda_1>0\right\}.
    \]

    Combining all together, we obtain
    \begin{align*}
        \min_{\A\in\R^{d\times d}}\P(\hat\y_\A\neq y)&=\int_{\bmu^\top\hat\bmu_s>0}\left(f(\hat\bmu_s)-f(-\hat\bmu)\right)Q\left(\frac{1}{\sigma}\right)+f(-\hat\bmu_s)d\hat\bmu_s\\
        &=\int_{\bmu^\top\hat\bmu_s>0}f(\hat\bmu_s)d\hat\bmu_s\cdot Q\left(\frac{1}{\sigma}\right)+\int_{\bmu^\top\hat\bmu_s<0}f(\hat\bmu_s)d\hat\bmu_s\cdot\left(1- Q\left(\frac{1}{\sigma}\right)\right)\\
        &=Q\left(-\frac{\sqrt{m}}{\sigma}\right)Q\left(\frac{1}{\sigma}\right)+Q\left(\frac{\sqrt{m}}{\sigma}\right)\left(1- Q\left(\frac{1}{\sigma}\right)\right)\\
        &=\left(1-Q\left(\frac{\sqrt{m}}{\sigma}\right)\right)Q\left(\frac{1}{\sigma}\right)+Q\left(\frac{\sqrt{m}}{\sigma}\right)\left(1- Q\left(\frac{1}{\sigma}\right)\right)\\
        &=Q\left(\frac{1}{\sigma}\right)+Q\left(\frac{\sqrt{m}}{\sigma}\right)-2Q\left(\frac{\sqrt{m}}{\sigma}\right)Q\left(\frac{1}{\sigma}\right).
    \end{align*}
    It completes the proof.
\end{proof}


\subsection{Proof of Theorem~\ref{thm non asymp}}\label{app sec non asymp}
\begin{proof}
Recap from Proposition~\ref{prop multilayer}. For any $L$-layer attention model with $L\geq2$, it can output
\begin{align}
f_{\att\mbar L}(\Z)=\x^\top(\X^\top\X/n-\sigma^2\Iden)\hat\bmu_s.\label{app non asymp output}
\end{align}
Let 
\[
\hat y=\sgn{f_{\att\mbar L}(\Z)}
\]
with $f_{\att\mbar L}(\Z)$ defined in \eqref{app non asymp output}. Then we have
\[
\P(y_{\att\mbar L}^\st(\Z)\neq y)\leq\P(\hat y\neq y).
\]
Therefore, in the following, we focus on upper-bounding the classification error $\P(\hat y\neq y)$ corresponding to \eqref{app non asymp output}. Given that the optimal prediction under the form $\sgn{\x^\top\A\hat\bmu_s}$ is given by $\hat y_{\bmu\bmu^\top}:=\sgn{\x^\top\bmu\bmu^\top\hat\bmu_s}$ (cf.~Theorem \ref{thm optimal A}), with its corresponding error presented in \eqref{n infty err}. To analyze the performance of $\hat{y}$, we study its difference from the prediction $\hat y_{\bmu\bmu^\top}$.

To begin with, let $\g_i=\bxi_i/\sigma\sim\Nc(0,\Iden)$ and $\g=\sum_{i=1}^n{\bxi_i}/\sigma\sqrt{n}\sim\Nc(0,\Iden)$. For simplicity, let $\A:=\X^\top\X/n-\sigma^2\Iden$. 
We get
\begin{align*}
    \A&=\frac{1}{n}\X^\top\X-\sigma^2\Iden\\&=\frac{1}{n}\left(\sum_{i=1}^n\bmu\bmu^\top+\bmu\bxi_i^\top+\bxi_i\bmu^\top+\bxi_i\bxi_i^\top\right)-\sigma^2\Iden\\
    &=\bmu\bmu^\top+\frac{\sigma}{\sqrt{n}}(\bmu\g^\top+\g\bmu^\top)+\sigma^2\left(\frac{\sum_{i=1}^n\g_i\g_i^\top}{n}-\Iden\right).
\end{align*}
Recall \eqref{app optim A conditioned err} from the proof of Theorem~\ref{thm optimal A}. Our goal is to bound
\[
\P(\hat y\neq y)=\E_{\hat\bmu}\left[Q\left(\frac{\bmu^\top\A\hat\bmu_s}{\sigma\tn{\A\hat\bmu_s}}\right)\right].
\]

Define
\begin{align}
\Bal:=\A-\bmu\bmu^\top=\frac{\sigma}{\sqrt{n}}(\bmu\g^\top+\g\bmu^\top)+\sigma^2\left(\frac{\sum_{i=1}^n\g_i\g_i^\top}{n}-\Iden\right).\label{def Bal}
\end{align}
From the Laurent-Massart inequality \citep{laurent2000adaptive}, we have that with probability at least $1-e^{-t_1}$ (assuming $t_1\geq d$), the first term of \eqref{def Bal} can be bounded by
\begin{align}
    &\frac{1}{\sqrt{n}}\left\|\bmu\g^\top+\g\bmu^\top\right\|\leq\frac{2\|\g\|}{\sqrt{n}}\leq 6\sqrt{\frac{t_1}{n}}.\label{app non asymp bound 1}
\end{align}
Additionally, from \cite{neopane2019lecture}, we have that with probability at least $1-e^{-t_{2}}$ (assuming $t_2\geq d$), the second term of $\Bal$ (cf. \eqref{def Bal}) is bounded by (with a universal constant $C>0$)
\begin{align}
&\left\|\frac{\sum_{i=1}^n\g_i\g_i^\top}{n}-\Iden\right\|\leq C\cdot\sqrt{\frac{t_2}{n}}.\label{app non asymp bound 2}
\end{align}
Combining \eqref{app non asymp bound 1} and \eqref{app non asymp bound 2}, we get with probability at least $1-2e^{-t}$ (for $t\geq d$)
\[
\|\Bal\|\leq C_1\sqrt{\frac{t}{n}}\quad\text{where}\quad C_1:=6\sigma+C\sigma^2.
\]

We also bound $\|\hat\bmu_s\|$ as follows. Let $\hat\bmu_s=\bmu+\sigma/\sqrt{m}\g'\sim\Nc(\bmu,\sigma^2m\Iden)$, similar to \eqref{app non asymp bound 1}, with probability at least $1-e^{-t_3}$ (assuming $2d\leq t_3\leq m/4\sigma^2$), we can bound
\[
\|\hat\bmu_s\|\leq 1+\frac{\sigma}{\sqrt{m}}\|\g'\|\leq 1+3\sigma\sqrt{\frac{t_3}{m}}\leq 3.
\]
Then consider a significantly large $n$ (to ensure that $\|\Bal\|\leq1/12$, e.g., $n\geq (12C_1)^2t$). With probability at least $1-3e^{-\min(t,t_3)}$ and suppose that $\bmu^\top\hat\bmu_s>0.5$, we can bound
\begin{align*}
\left|\frac{\bmu^\top\A\hat\bmu_s}{\tn{\A\hat\bmu_s}}-\frac{\bmu^\top\bmu\bmu^\top\hat\bmu_s}{\tn{\bmu\bmu^\top\hat\bmu_s}}\right|&=\left|\frac{\bmu^\top(\Bal+\bmu\bmu^\top)\hat\bmu_s}{\tn{(\Bal+\bmu\bmu^\top)\hat\bmu_s}}-\frac{\bmu^\top\bmu\bmu^\top\hat\bmu_s}{\tn{\bmu\bmu^\top\hat\bmu_s}}\right|\\
&\leq\left|\frac{\bmu^\top\Bal\hat\bmu_s}{\min(\tn{(\Bal+\bmu\bmu^\top)\hat\bmu_s},\tn{\bmu\bmu^\top\hat\bmu_s})}\right| \\
&\leq\frac{\|\Bal\|\cdot\|\hat\bmu_s\|}{\bmu^\top\hat\bmu_s-\|\Bal\|\cdot\|\hat\bmu_s\|}\\
&\leq4\|\Bal\|\cdot\|\hat\bmu_s\|\\
&\leq C_2\sqrt{\frac{t}{n}}\quad\text{where}\quad C_2:=12 C_1.
\end{align*}

Now, we are ready to bound the classification error, where we get
\begin{align*}
\P\left(\hat y\neq y\right)&=\E_{\hat\bmu}\left[Q\left(\frac{\bmu^\top\A\hat\bmu_s}{\sigma\tn{\A\hat\bmu_s}}\right)\right]\\
&=\E_{\hat\bmu}\left[Q\left(\frac{1}{\sigma}+\frac{1}{\sigma}\left(\frac{\bmu^\top\A\hat\bmu_s}{\tn{\A\hat\bmu_s}}-\frac{\bmu^\top\bmu\bmu^\top\hat\bmu_s}{\tn{\bmu\bmu^\top\hat\bmu_s}}\right)\right)\right]\\
&\leq \P(\bmu^\top\hat\bmu_s>0.5)\left( Q\left(\frac{1-C_2\sqrt{t/n}}{\sigma}\right)+3e^{-\min(t,t_3)}\right)+\P(\bmu^\top\hat\bmu_s<0.5)\\
&\leq Q\left(\frac{1-C_2\sqrt{t/n}}{\sigma}\right)+3e^{-\min(t,t_3)}+Q\left(\frac{\sqrt{m}}{2\sigma}\right).
\end{align*}
Choosing $t=t_3=2d$, since $m/4\sigma^2\geq 2d$, we obtain
\begin{align*}
    \P\left(\hat y\neq y\right)&\leq Q\left(\frac{1-C_2\sqrt{2d/n}}{\sigma}\right)+3e^{-2d}+0.5e^{-d}\\
    &\leq Q\left(\frac{1-C_2\sqrt{2d/n}}{\sigma}\right)+e^{-d}.
\end{align*}
It completes the proof.
\end{proof}

\end{document}